%% file: arxiv_submission.tex
\documentclass[11pt]{article}

\usepackage[utf8]{inputenc} 
\usepackage{geometry} 
\usepackage{fancyhdr} 
\usepackage{amsthm}
\usepackage[colorlinks]{hyperref} 
\usepackage{subfig}
\usepackage{bbm}

\usepackage[utf8]{inputenc} 
\usepackage[T1]{fontenc}
\usepackage{url}
\usepackage{ifthen}
\usepackage{cite}
\usepackage{amsfonts}
\usepackage{hhline}
\usepackage{bm}
\usepackage{algorithm,algorithmic}
\usepackage{color,soul}
\usepackage[cmex10]{amsmath} 
\usepackage{graphicx}
\graphicspath{ {graphs/} }
\newcommand{\E}{\ensuremath{\mathbb{E}}}
\newcommand{\bbP}{\ensuremath{\mathbb{P}}}
\newcommand{\bbQ}{\ensuremath{\mathbb{Q}}}
\newcommand{\bbS}{\ensuremath{\mathbb{S}}}
\newcommand{\Q}{\ensuremath{\mathbb{Q}}}
\renewcommand{\S}{\ensuremath{\mathbb{S}}}
\renewcommand{\P}{\ensuremath{\mathbb{P}}}
\newcommand{\bbE}{\ensuremath{\mathbb{E}}}
\newcommand{\cO}{\ensuremath{\mathcal{O}}}
\newcommand{\bi}{\ensuremath{\boldsymbol{i}}}
\newcommand{\by}{\ensuremath{\mathbf{y}}}
\newcommand{\cX}{\ensuremath{\mathcal{X}}}
\newcommand{\cA}{\ensuremath{\mathcal{A}}}
\newcommand{\cF}{\ensuremath{\mathcal{F}}}
\newcommand{\Ni}{\ensuremath{N^{\text{in}}}}
\newcommand{\No}{\ensuremath{N^{\text{out}}}}
\newcommand{\lc}{\ensuremath{\lfloor}}
\newcommand{\rc}{\ensuremath{\rfloor}}
\newcommand{\kl}{\ensuremath{D_{KL}}}

\newtheorem{theorem}{Theorem}

\newtheorem{lemma}{Lemma}

\newtheorem{corollary}{Corollary}
\newtheorem{definition}{Definition}

\begin{document}

\begin{center}

	{\bf{\LARGE{Online learning with graph-structured feedback against adaptive adversaries}}}

	\vspace*{.25in}

	\begin{tabular}{ccc}
		{\large{Zhili Feng$^*$}} & \hspace*{.75in} & {\large{Po-Ling Loh$^{\dagger}$}}\\
		{\large{\texttt{zfeng49@cs.wisc.edu}}} & \hspace*{.75in} & {\large{\texttt{loh@ece.wisc.edu}}}
			\end{tabular}
\begin{center}
Department of Computer Science$^*$\\
Departments of Electrical \& Computer Engineering and Statistics$^\dagger$\\
University of Wisconsin - Madison\\
Madison, WI 53706
\end{center}

\vspace*{.2in}

April 2018

\vspace*{.2in}

\end{center}

\begin{abstract}
We derive upper and lower bounds for the policy regret of $T$-round online learning problems with graph-structured feedback, where the adversary is nonoblivious but assumed to have a bounded memory. We obtain upper bounds of $\widetilde\cO(T^{2/3})$ and $\widetilde\cO(T^{3/4})$ for strongly-observable and weakly-observable graphs,  respectively, based on analyzing a variant of the Exp3 algorithm. When the adversary is allowed a bounded memory of size 1, we show that a matching lower bound of $\widetilde\Omega(T^{2/3})$ is achieved in the case of full-information feedback. We also study the particular loss structure of an oblivious adversary with switching costs, and show that in such a setting, non-revealing strongly-observable feedback graphs achieve a lower bound of $\widetilde\Omega(T^{2/3})$, as well.
\end{abstract}



\input{intro}
\input{background_and_prelims.tex}
\input{graph_upper_bound.tex}
\input{lower_bound.tex}
\input{revealing_act.tex}

\section{Discussion}
\label{SecDiscussion}

In this paper, we have studied the policy regret of online learning problems with various types of graph-structured feedback. We have shown that when the adversary is allowed to be nonoblivious, the sharp characterization of minimax regret in terms of strong or weak observability becomes somewhat more complicated than in the oblivious case. In particular, we have shown that a mini-batched version of the Exp3.G algorithm leads to $\widetilde\cO(T^{2/3})$ regret in the strongly-observable case and $\widetilde\cO(T^{3/4})$ regret in the weakly-observable case when the adversary has bounded memory, but strongly-observable feedback graphs exist with minimax regret of both $\Theta(T^{1/2})$ and $\widetilde\Theta(T^{2/3})$, for the class of adversaries with switching costs. We have also established a strategy that achieves $\cO(T^{2/3})$ regret for certain weakly-observable graphs with switching costs, leaving open the possibility of weakly-observable graphs being subdivided into various hardness classes, as well.

Existing results~\cite{piccolboni2001discrete}, \cite{bartok2010toward}, \cite{antos2013toward}, and \cite{bartok2014partial} show that online learning games against oblivious opponents fall into one of four categories: trivial games with $0$ regret, easy games with $\widetilde\Theta(\sqrt{T})$ regret, hard games with $\widetilde\Theta(T^{2/3})$ regret, and hopeless games with $\Omega(T)$ regret. 
This does not rule out the possibility that the $\widetilde\cO(T^{3/4})$ bound is tight for certain classes of weakly-observable graphs, since we are studying the fundamentally different setting of nonoblivious adversaries. An important open question is to determine whether feedback graphs actually exist that produce a $\Omega(T^{3/4})$ lower bound.

Another open question is to characterize the minimax regret for online learning problems with revealing strongly-observable graphs for switching cost adversaries. Our current results contain a gap between the $\widetilde\cO(T^{2/3})$ upper bound and $\widetilde\Omega(T^{1/2})$ lower bound. For full-information games with switching costs, the \emph{Follow the Lazy Leader} (FLL) algorithm~\cite{kalai2005efficient} and \emph{Shrinking Dartboard} (SD) algorithm~\cite{geulen2010regret} are known to achieve $\cO(\sqrt{T})$ regret, and it may be possible to extend such strategies to other strongly-observable games, as well.

\enlargethispage{-0.3cm} 

\bibliographystyle{IEEEtran}
\bibliography{citation}


\newpage

\appendix

\input{appendices/appendix_a.tex}
\input{appendices/appendix_b.tex}

\input{appendices/appendix_c.tex}
\input{appendices/appendix_d.tex}

\end{document}

%% file: intro.tex
\section{Introduction}

The canonical setting of online learning involves a repeated game between a player and an adversary~\cite{cesa2006prediction}. At each round, the player chooses an action in the action space, and the adversary reveals the losses (equivalently, rewards) corresponding to the action. Such games may be characterized via their forms of feedback, two types of which are particularly popular: either the player observes the losses of all actions, which is known as the \emph{full-information game}; or the player only observes the loss of the chosen action, which is known as the \emph{multi-armed bandit problem}. More generally, we may consider the notion of \emph{graph-structured feedback} introduced by Mannor et al.~\cite{mannor2011bandits}. A feedback graph $G=(V, E)$ is a directed graph where each node $i\in V$ represents an action, and an edge $(i, j)\in E$ means the player observes the loss of action $j$ when choosing action $i$. Accordingly, full-information feedback is represented by a complete graph with self-loops and bandit feedback is represented by a graph with only self-loops (see Figure~\ref{fig:feedback_graphs}). Other settings include the \emph{apple tasting} problem~\cite{helmbold1992apple} and the \emph{revealing action} game. The goal of the player in a $T$-round online learning game is to attain order $o(T)$ regret, in which case the player is considered to be ``learning.''
\begin{figure}[htbp]
\centering
\subfloat[Apple Tasting]{\includegraphics[width=0.3\linewidth]{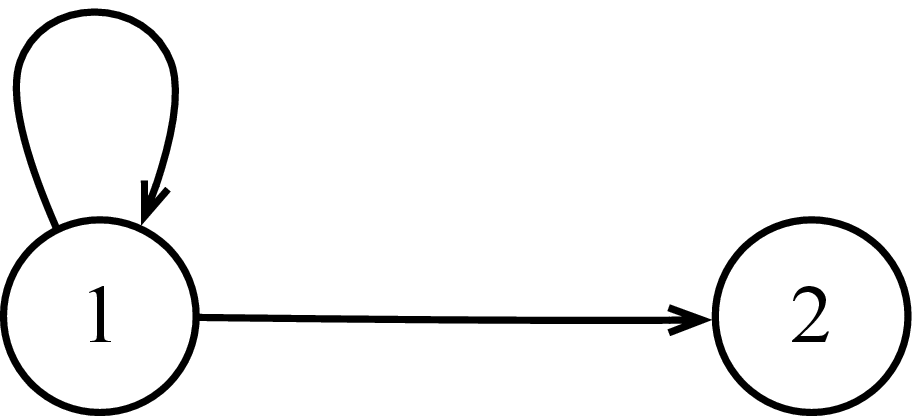}
\label{fig:apple_tasting}}\hfil
\subfloat[Two-armed Bandit]{\includegraphics[width=0.3\linewidth]{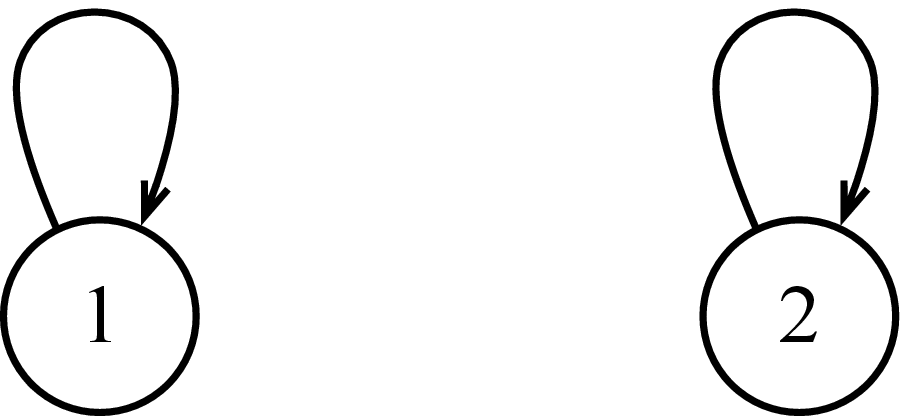}
\label{fig:bandit}}

\subfloat[Full-information Game]{\includegraphics[width=0.3\linewidth]{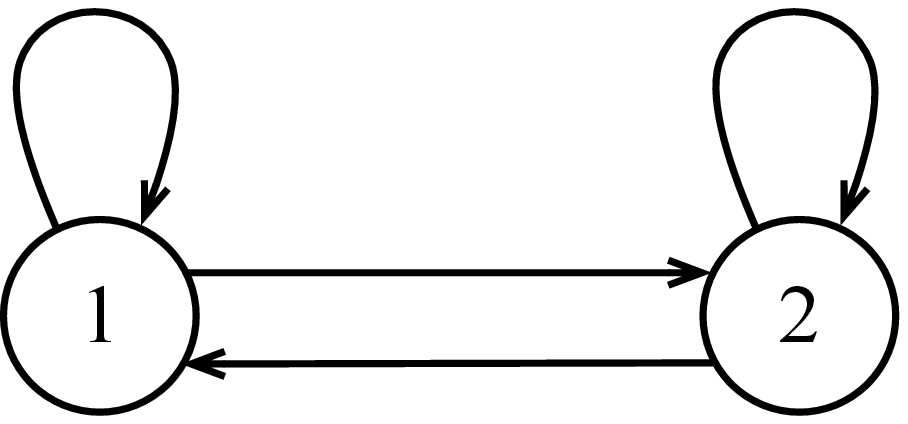}
\label{fig:full_info}}\hfil
\subfloat[Revealing Action Game]{\includegraphics[width=0.3\linewidth]{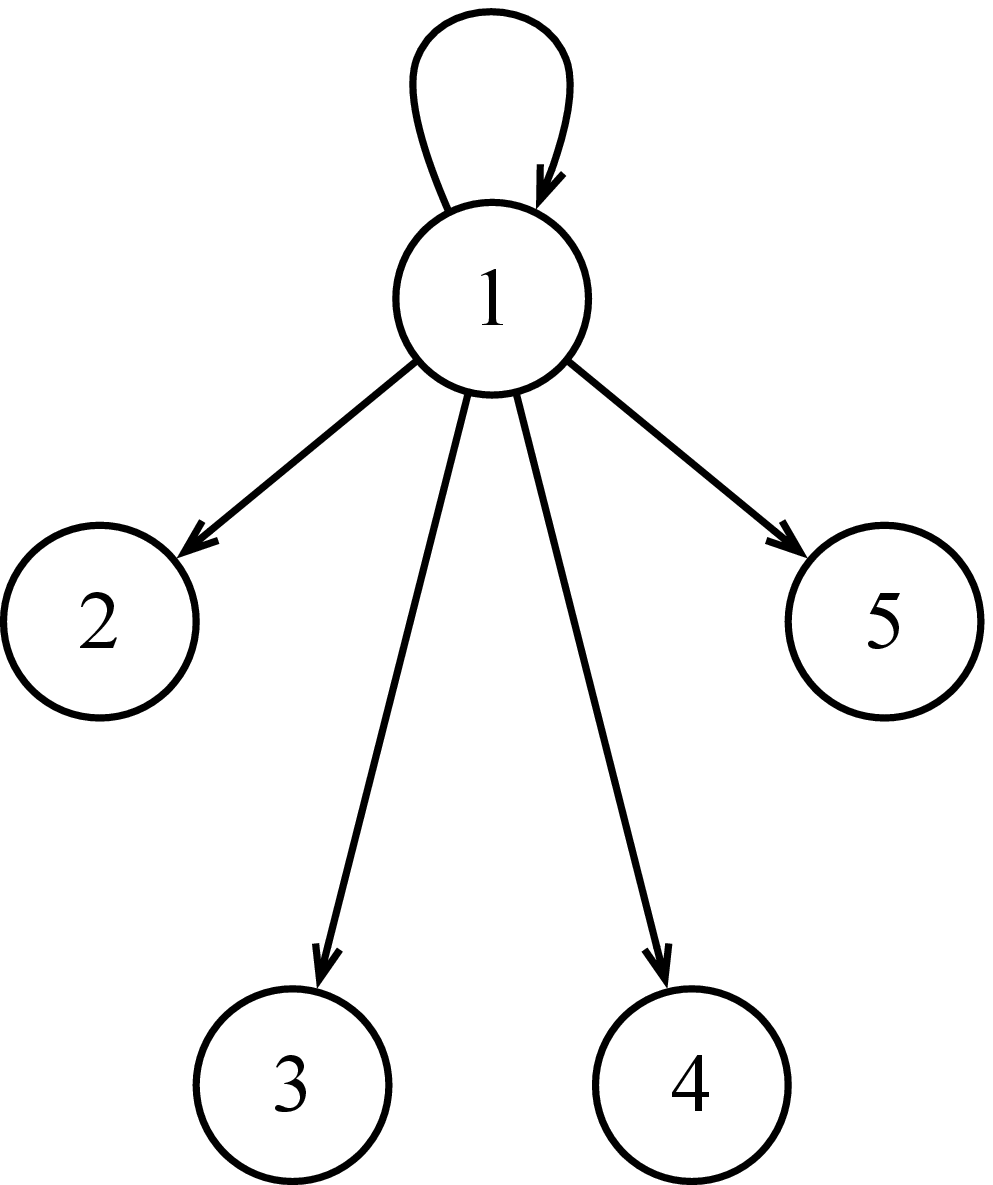}
\label{fig:revealing_action}}
\caption{Examples of feedback graphs}
\label{fig:feedback_graphs}
\end{figure}

Another important characterization of online games is the dependency of the adversary's loss functions on the player's actions. If an adversary's losses do not depend on the player's past actions, it is known as an \emph{oblivious adversary}; otherwise, it is called \emph{adaptive} or \emph{nonoblivious}. Arora et al.~\cite{arora2012online} showed that if the adversary's strategy is allowed to depend on an arbitrary number of previous actions, the minimax regret---defined as the regret when both the player and the adversary behave optimally---becomes $\Omega(T)$. Hence, a weaker adversary is required for learnability. The same paper demonstrated that mini-batching an algorithm with $\cO(T^{q})$ regret for oblivious adversaries leads to $\cO(T^{1/(2-q)})$ policy regret for nonoblivious adversaries with bounded memory. A specific class of unit-memory adversaries of particular interest corresponds to oblivious adversaries with switching costs.
Although the minimax regret was shown to be $\Theta(T^{1/2})$ in the case of full-information games and $\widetilde\Theta(T^{2/3})$ in the case of bandit feedback~\cite{cesa2013online, dekel2014bandits}, the gap between $\widetilde\cO(T^{2/3})$ upper bounds and $\Omega(T^{1/2})$ lower bounds for the more general class of adversaries with unit memory in the case of full-information feedback has remained unaddressed.

For the problem of general feedback graphs with oblivious adversaries, Alon et al.~\cite{alon2015online, alon2017nonstochastic} showed that the regret is characterized by certain characteristics of the graph structure involving domination numbers and independent sets. This leads to three different regret regimes for minimax regret: $\widetilde\Theta(T^{1/2})$, $\widetilde\Theta(T^{2/3})$, and $\Theta(T)$, which may be compared with the different rates of learning for partial monitoring games~\cite{bartok2014partial}.
The goal of this paper is to analyze the policy regret for online games with graph-structured feedback when the adversary is nonoblivious.

We make the following contributions:
\begin{itemize}
\item In the case of bounded-memory adversaries, we show that a mini-batched version of the Exp3.G algorithm achieves $\widetilde\cO(T^{2/3})$ regret for strongly-observable graphs and $\widetilde\cO(T^{3/4})$ regret for weakly-observable graphs.
\item For adversaries with bounded memory of size 1, we derive a lower bound of $\widetilde\Omega(T^{2/3})$ for full-information games, closing a gap in the current literature. Consequently, if we restrict our attention to strongly-observable graphs and adversaries with memory of size $1$, the minimax regret is $\widetilde\Theta(T^{2/3})$. 
\item For oblivious adversaries with switching costs, we derive a lower bound of $\widetilde\Omega(T^{2/3})$ for non-revealing strongly-observable graphs, showing that the minimax regret is $\widetilde\Theta(T^{2/3})$ for these classes of games.
\item In the case of a weakly-observable graph corresponding to a revealing action game, where the adversary is again oblivious with switching costs, we show that the minimax regret is $\Theta(T^{2/3})$.
\end{itemize}
Our contributions are summarized in Table~\ref{tab:sota_summary} and Figure~\ref{fig:game_difficulty} of Appendix~\ref{AppSummary}, and highlighted in boldface.

The remainder of the paper is organized as follows: In Section~\ref{SecBackground}, we provide background and notation for online learning with adaptive adversaries and feedback graphs. In Sections~\ref{sec:upper_bound} and~\ref{SecLower}, we derive our upper and lower bounds on policy regret. In Section~\ref{sec:reveal_act}, we consider a special revealing action game. We conclude in Section~\ref{SecDiscussion} by discussing open questions related to our work. For proof details, we refer the reader to the extended version of the paper.

\textbf{Notation:} We write $A_{1:n}$ to denote the sequence $(A_1, A_2, \ldots, A_n)$. When $m$ is a positive integer, we write $[m]$ to denote the sequence $1,2,\ldots, m$. We write $\bi^t$ to denote $t$ copies of a fixed action $i$.

%% file: background_and_prelims.tex
\section{Background and preliminaries}
\label{SecBackground}

We begin by formalizing some notation and reviewing the notion of policy regret. We then introduce several graph-theoretic notions arising in the setting of feedback graphs.

\subsection{Types of regret}

Consider a $T$-round game with an oblivious adversary, and denote the space of possible actions by $\mathcal{X}$. Denote the player's action at time $t$ by $X_t$, and denote the loss function chosen by the adversary by $f_t: \cX\mapsto[0,1]$. The \emph{standard regret} is then defined as follows:
\[
R_T^{\text{std}} = \sum_{t=1}^Tf_t(X_t)-\min_{x\in\cX}\sum_{t=1}^Tf_t(x).
\]
In other words, the regret compares the cumulative loss of the player's actual actions to the cumulative loss of the best fixed action in hindsight.
If $\mathcal{F}$ denotes the space of all oblivious loss sequences, the player seeks to minimize $\sup_{f_{1:T}\subseteq \mathcal{F}}\bbE[R_T^{\text{std}}]$, where the expectation is taken with respect to any possible randomness in the player's strategy. Hence, the difficulty of a game may be characterized by the \emph{minimax regret}: If $\cA$ denotes the class of strategies available to the player, the minimax regret is defined as $\inf_{\cA}\sup_{\mathcal{F}}\bbE[R_T^{\text{std}}]$.


When the adversary is allowed to adapt to the player's actions, we use a slightly different notion of regret. In such a setting, the loss functions determined by the adversary may depend on past actions of the player, which we denote by the functions $f_t:\mathcal{X}^t\mapsto[0,1]$. The best fixed action of the player may incur a different loss sequence than the sequence of loss functions encountered by the player in a strategy that switches between actions. Accordingly, Arora et al.~\cite{arora2012online} introduced the notion of \emph{policy regret}:
\[
R_T = \sum_{t=1}^Tf_t(X_{1:t})-\min_{x\in\cX}\sum_{t=1}^Tf_t(\underbrace{x,\dots, x}_\text{$t$ copies}).
\]
In this paper, we will generally use the term ``regret" to refer to policy regret, unless explicitly stated otherwise. Accordingly, we wish to characterize the quantity
\begin{equation}
\label{EqnMinmaxReg}
\inf_{\cA}\sup_{\mathcal{F}}\bbE[R_T].
\end{equation}
Note that $R_T = R_T^{\text{std}}$ when the adversary is oblivious, since $f_t(X_{1:t}) = f_t(X_t)$ and $f_t(x, \dots, x) = f_t(x)$. Hence, any lower bound on the standard regret for oblivious adversaries translates into a lower bound on the policy regret for oblivious (and nonoblivious) adversaries.

We will be particularly interested in the subclass of adaptive adversaries with bounded memory. If the loss function can only depend on the $m+1$ most recent actions of the player, the adversary is called \emph{adaptive with bounded memory of size $m$}. In other words, the loss functions take the following form:
\[
    f_t(X_1, \ldots, X_{t})=f_t(X_1', \ldots, X_{t-m-1}', X_{t-m}, \ldots, X_{t}),
\]
for any $X_{1:t}\in \cX^t$ and $X_{1:(t-m-1)}'\in\cX^{t-m-1}$. Note that if the loss function is oblivious, corresponding to an adversary with memory of size 0, we have
\[
    f_t(X_1, \ldots, X_{t})=f_t(X_1', \ldots, X_{t-1}', X_{t}).
\]
A further subclass of adaptive adversaries with bounded memory of size $1$ is the class of \emph{oblivious adversary with switching costs}, obtained by first picking a sequence of oblivious losses $\{\ell_t\}$, and then defining the overall loss sequence as
\begin{align}
\begin{split}
    &f_1(X_1)=\ell_1(X_1), \\
    \forall t\geq 2, \quad &f_t(X_{t-1}, X_{t})=\ell_t(X_t)+\mathbbm{1}_{\{X_t\neq X_{t-1}\}}.
\end{split}
\end{align}


\subsection{Feedback graphs}

Given a feedback graph $G=(V,E)$, we define \mbox{$\Ni(v)=\{w: (w, v)\in E\}$} and $\No(v)=\{w: (v, w)\in E\}$, for all $v\in V$.

Alon et al.~\cite{alon2015online} introduced two types of feedback graphs in their study of oblivious adversaries, which we also adopt in this paper. We review two important definitions:

\begin{definition}[Observability]
Given graph $G=(V,E)$,
\begin{enumerate}
\item  A node $v\in V$ is \emph{observable} if $\Ni(v)\neq \emptyset$.
\item $v\in V$ is \emph{strongly-observable} if either (i) $\{v\}\subseteq \Ni(v)$, (ii) $V\backslash\{v\}\subseteq \Ni(v)$, or both.
\item A node that is observable but not strongly-observable is called \emph{weakly-observable}.
\item A graph $G$ is observable if all its vertices are observable, and strongly-observable if all its vertices are strongly-observable. A graph is weakly-observable if it is observable, but not strongly.
\end{enumerate}
\end{definition}
We call an online learning problem \emph{strongly-observable} (respectively, \emph{weakly-observable}) if the feedback graph is strongly-observable (respectively, weakly-observable). Throughout the paper, we assume that the player knows the graph structure, and the structure remains unchanged. Thus, when deriving bounds on the minimax regret~\eqref{EqnMinmaxReg}, we may assume that the space of loss sequences is defined over a fixed feedback graph, rather than a class of potential graphs.

\begin{definition} [Weak Domination]
Given a directed graph $G=(V, E)$ with a set of weakly-observable nodes $U\subseteq V$, a \emph{weakly dominating set} $D\subseteq V$ is a set such that $\forall v\in U$, $\exists d\in D$ such that $v\in \No(d)$. The \emph{weak domination number} $\delta=\delta(G)$ is the size of the smallest weakly dominating set.
\end{definition}

Finally, we introduce a concept of revealability:
\begin{definition} [Revealability]
We call a strongly-observable graph $G=(V, E)$ \emph{revealing} if $\Ni(u)\cap \Ni(v)\neq \emptyset$ for all $u,v \in V$. If $G$ is not revealing, we call $G$ \emph{non-revealing}.
\end{definition}

%% file: graph_upper_bound.tex
\section{Upper Bounds}
\label{sec:upper_bound}

In this section, we derive upper bounds for policy regret with strongly-observable and weakly-observable feedback graphs. The main idea is to create a mini-batched version of the Exp3.G algorithm of Alon et al.~\cite{alon2015online} (stated as Algorithm~\ref{alg:exp3g} in Appendix~\ref{AppAlgorithms}), using a technique of Arora et al.~\cite{arora2012online}: Rounds are partitioned into batches of length $\tau$, and for each batch, an action is selected according to Exp3.G and played $\tau$ times. Average losses are then fed back to the Exp3.G algorithm.

\begin{algorithm}[]
    \caption{Mini-batched Exp3.G}
    \begin{algorithmic}[1]
    \renewcommand{\algorithmicrequire}{\textbf{Input:}}
    \renewcommand{\algorithmicensure}{\textbf{Output:}}
    \REQUIRE Time horizon $T$, adversary's memory size $m$, mini-batch size $\tau>m$, Exp3.G algorithm $\mathcal{A}$
    \ENSURE A sequence of actions $X_{1:T}$
    \STATE Set $J=\lc T/\tau \rc$
    \FOR {$j=1,2,\ldots, J$}
        \STATE Use $\mathcal{A}$ to choose action $Z_j$, set $X_{(j-1)\tau+1:j\tau}=Z_j$
        \FOR {$t=1,2,\ldots, \tau$}
            \STATE Play $Z_j$
        \ENDFOR
        \STATE Gather loss $\frac{1}{\tau}\sum_{t=(j-1)\tau+1}^{j\tau}f_t(X_{1:t})$ and feed to $\mathcal{A}$
    \ENDFOR
    \FOR {$t=J\tau+1,\ldots, T$}
        \STATE Use $\mathcal{A}$ to choose action $X_t$
    \ENDFOR
    \RETURN $X_{1:T}$
    \end{algorithmic} 
    \label{alg:mini_batch_exp3g}
\end{algorithm}


Alon et al.~\cite{alon2015online} proved that the Exp3.G algorithm obtains $\widetilde\cO(\sqrt{T})$ regret for strongly-observable feedback graphs and $\widetilde\cO(T^{2/3})$ regret for weakly-observable graphs. However, these bounds are obtained against oblivious loss sequences. In order to apply the result of Arora et al.~\cite{arora2012online} (cf.\ Lemma~\ref{thm:policyregretreduction} in Appendix~\ref{AppUpper}), we first need to modify these bounds to adaptive opponents. This is stated in the following theorem. Recall that the \emph{independence number} of a graph is the cardinality of the largest subset of vertices that are not connected by any edges.

\begin{theorem}\label{thm:exp3gupperbound}
Let $G=(V,E)$ be a feedback graph with $K=|V|$, independence number $\alpha=\alpha(G)$, and weakly dominating number $\delta=\delta(G)$. Let $D$ be a weakly dominating set such that $|D|=\delta$. The expected standard regret $\E[R_T^{\text{std}}]$ of the Exp3.G algorithm against any adaptive adversary satisfies the following:
\begin{enumerate}
\item If $G$ is strongly-observable, then for $U=V$, $\gamma=\min\left\{ \left(\frac{1}{\alpha T}\right)^{1/2}, \frac{1}{2} \right\}$, and $\eta=\frac{1}{2}\gamma$, the expected standard regret against any adaptive loss sequence is $\mathcal{O}\left(\alpha^{1/2}T^{1/2}\ln(KT)\right)$.
\item If $G$ is weakly-observable and $T\geq K^3\ln(K)/\delta^2$, then for $U=D$, $\gamma=\min\left\{\left(\frac{\delta\ln K}{T}\right)^{1/3}, \frac{1}{2}\right\}$, and $\eta=\frac{\gamma^2}{\delta}$, the expected standard regret against any adaptive loss sequence is $\mathcal{O}\left((\delta\ln K)^{1/3}T^{2/3}\right)$.
\end{enumerate}
\end{theorem}
The proof of Theorem~\ref{thm:exp3gupperbound} is provided in Appendix~\ref{AppThmExp3}. Combining Lemma~\ref{thm:policyregretreduction} in Appendix~\ref{AppUpper} with Theorem~\ref{thm:exp3gupperbound}, we obtain the desired upper bound:
\begin{theorem}
The mini-batched version of Exp3.G (Algorithm~\ref{alg:mini_batch_exp3g}) against adversaries with memory of size $m$ achieves $\widetilde\cO(T^{2/3})$ policy regret if $G$ is strongly-observable, and $\widetilde\cO(T^{3/4})$ policy regret if $G$ is weakly-observable.
\end{theorem}

The upper bound holds for any adaptive adversary with bounded memory, indicating that online learning problems do not become harder even if the adversary's memory is bounded by a larger constant. We show that this upper bound may not always be tight: For example, for oblivious adversaries with switching costs, we will show in Section~\ref{sec:reveal_act} that a revealing action feedback graph leads to a minimax regret bound of $\Theta(T^{2/3})$, even though the graph is weakly-observable. 

%% file: lower_bound.tex
\section{Lower bounds}
\label{SecLower}

We now turn to lower bounds. Although we cannot obtain matching lower bounds for all classes of  adversaries and feedback graphs, we show that for certain types of graphs and adversaries, our upper bounds are tight. From Alon et al.~\cite{alon2015online}, the standard regret for any observable graph is lower-bounded by $\widetilde\Omega(T^{1/2})$ when the adversary is oblivious. Thus, we certainly have a policy regret lower bound of $\widetilde\Omega(T^{1/2})$ for all observable graphs in the case of nonoblivious adversaries. The results of this section show that the gap between upper and lower bounds can be closed in certain special cases.


\subsection{Adversaries with bounded memory of size 1}

In this subsection, we show that the full-information game against an adversary with bounded memory of size 1 has $\widetilde\Omega(T^{2/3})$ regret, thus closing a gap in the literature~\cite{cesa2013online}. It suffices to consider an easy setup with two arms; bounds for the more general case with arbitrarily many arms may be derived using a similar technique.

\begin{theorem}\label{thm:full_info_lower_bound}
For a time horizon $T>10$ and any $(T+1)$-round online learning problem with full-information feedback, and for any randomized player strategy, there exists a bounded loss sequence $f_1,\ldots, f_{T+1}$ with memory of size $1$ such that
\[
\bbE[R_{T+1}]\geq\frac{T^{2/3}}{500\log_2T}=\widetilde\Omega(T^{2/3}).
\]
\end{theorem}

\begin{proof} (Sketch.)
The proof begins by choosing appropriate oblivious loss sequences $L_1(X_1),L_2(X_2),\ldots$ in a similar fashion as in Dekel et al.~\cite{dekel2014bandits}, upon which we build our adaptive loss sequences with memory of size $1$. Formally, Algorithm~\ref{alg:loss_seq_generation}  in Appendix~\ref{AppMRW} is used to generate the oblivious loss sequences $L_{1:T}$. We then define our adaptive loss sequence as follows:
\begin{align}\label{eq:mem_1_loss}
\begin{split}
    &f_1({}\boldsymbol{\cdot}{}, X_1)=0,\\
    &f_t(X_{t-1}, X_t)=L_{t-1}(X_{t-1})+\mathbbm{1}_{\{X_{t-1}\neq X_t\}}, \\
    &\quad \text{for }t=2,\ldots, T+1.
\end{split}
\end{align}
The proof proceeds via a KL-divergence calculation; details are provided in Appendix~\ref{AppThmFullInfo}.
\end{proof}


Since full-information feedback graphs constitute the ``easiest" games with strongly-observable feedback graphs, the lower bound extends to all strongly-observable games against unit-memory adversaries:
\begin{corollary}
For any strongly-observable feedback graph $G=(V, E)$, the regret is lower-bounded by $\widetilde\Omega(T^{2/3})$ when the memory of the adversary is bounded by 1.
\end{corollary}


\subsection{Non-revealing strongly-observable games with switching costs}
\label{sec:strong_observable_lower_bound}

We now focus on strongly-observable games where the adversary is oblivious with switching costs. Unfortunately, although the Exp3.G algorithm is minimax optimal for all strongly-observable graphs when the adversary is oblivious, this is not true when the game involves switching costs. It is known that certain strategies exist which incur $\cO(\sqrt{T})$ regret in the full-information game, whereas the multi-armed bandit problem suffers a lower bound of $\widetilde\Omega(T^{2/3})$~\cite{dekel2014bandits}, even though both games are induced by strongly-observable graphs. Hence, some strongly-observable games are more difficult than others.

The proof is based on a reduction from the original graph to a subgraph, and we use the max-min inequality to show that the game induced by the original graph is at least as hard as the game induced by the subgraph. Accordingly, we will use the following notion in our development:

\begin{definition} [Observability preserving property]
\label{def:observability_preserve}
Let $G_1=(V_1, E_1)$ be a subgraph of $G=(V, E)$. Let $G_2=(V_2, E_2)$ be such that $V_1\cup V_2=V$ and $E_1\cup E_2\subseteq E$. We say that $G_1$ \emph{preserves the observability} of $G$ if and only if
\begin{align*}
& \forall v\in V_2,\ \exists w\in V_1 \text{ s.t } \\
& \qquad \forall b\in V_1, \text{ if } (v, b)\in E, \text{ then } (w, b)\in E_1.
\end{align*}
Note that it is possible that $w=b$. We call $w$ an \emph{observing node} of $v$ in $G_1$, and write $w\overset{\Delta}{=}v_{ob}$.
\end{definition}

We have the following result:

\begin{theorem}
\label{ThmSC}
If the feedback graph $G=(V,E)$ is strongly-observable and non-revealing and the adversary is oblivious with switching costs, the expected regret of any player strategy is bounded below by $\widetilde\Omega(T^{2/3})$.
\end{theorem}

\begin{proof}
Recall the notion of revealability defined in Section~\ref{SecBackground}. We first show that the independence number of $G$ is at least 2 for a non-revealing strongly-observable graph. This is proved in Lemma~\ref{lem:ind_num} in Appendix~\ref{AppSC}.

Hence, if $G$ is non-revealing, we can find $u, v\in V$ such that $u$ and $v$ are independent and their dominating sets $\Ni(u)$ and $\Ni(v)$ are disjoint. By the enumeration in Figure~\ref{fig:strongobservablecases} of Appendix~\ref{AppLower}, such a graph must contain a two-node subgraph $G_1$ that preserves the observability of $G$. By Lemma~\ref{graphreductionlma} in Appendix~\ref{AppSC}, the game induced by $G$ is at least as hard as any game on $G_1$. The game on $G_1$ is simply a bandit problem, so it has $\widetilde\Omega(T^{2/3})$ regret~\cite{dekel2014bandits}.
\end{proof}


%% file: revealing_act.tex
\section{The revealing action game}
\label{sec:reveal_act}

Section~\ref{sec:upper_bound} supplies a policy regret upper bound for a mini-batched version of Exp3.G, where the player's policy regret is $\widetilde\cO(T^{3/4})$ if the feedback graph is weakly-observable. The $\widetilde\Omega(T^{2/3})$ regret lower bound for weakly-observable games against oblivious opponents naturally extends to games against adaptive opponents. However, which bound is improvable?

The answer turns out to be the upper bound: In the revealing action game, better algorithms exist. We first consider the \emph{label-efficient prediction} problem~\cite{helmbold1997some}, which is nearly the same as a revealing action game. The difference is that in the revealing action game, the revealing action $r$ is a vertex in the graph $G$, and upon playing that action, the player will achieve loss $\ell_t(r)$ and also observes the losses of all the vertices in $G$, including $r$. On the other hand, in the label-efficient prediction problem, the ``revealing action" is not an actual action that the player can play; instead, the player chooses this action only to reveal the losses of other actions, and wants to query this action as infrequently as possible. Cesa-Bianchi et al.~\cite{cesa2005minimizing} devised a lazy player strategy for such problems, provided in Algorithm~\ref{alg:lazy_label_efficient} in Appendix~\ref{AppAlgorithms}. The authors derived the following regret bound:
\begin{lemma}
\label{LemRTLazy}
Fix a time horizon $T$. Denote the number of actions by $N$. Set $\epsilon=\frac{m}{T}$ and $\eta=\frac{\sqrt{2m\ln N}}{T}$ in Algorithm~\ref{alg:lazy_label_efficient}. Then for any oblivious loss sequences, the expected regret satisfies
\begin{equation*}
\bbE[R_T^\text{std}]\leq T\sqrt{\frac{2\ln N}{m}}.
\end{equation*}
\end{lemma}

Taking $m=T^{2/3}$ in Lemma~\ref{LemRTLazy} yields the bound $\cO(T^{2/3})$. Notice $m$ is the expected number of queries. Furthermore, Cesa-Bianchi et al.~\cite{cesa2006regret} provided a non-lazy algorithm for the revealing action game. Combining these ideas, we obtain a lazy algorithm that achieves $\cO(T^{2/3})$ regret against oblivious adversaries with switching costs, described in Algorithm~\ref{alg:lazy_revealing_action_alg}. We have the following result:

\begin{theorem}
Fix a time horizon $T$. Let $N$ denote the number of actions. For $\epsilon = \frac{m}{T}$ and $\eta=\frac{\sqrt{2m\ln N}}{T}$,  Algorithm~\ref{alg:lazy_revealing_action_alg} achieves $\cO(T^{2/3})$ policy regret in the revealing action game against oblivious opponents with switching costs.
\end{theorem}

\begin{proof}
Notice that the difference between Algorithm~\ref{alg:lazy_revealing_action_alg} and Algorithm~\ref{alg:lazy_label_efficient} is that Algorithm~\ref{alg:lazy_revealing_action_alg} needs to actually choose the revealing action $r$ and suffer loss $\ell_t(r)$ at round $t$, whereas Algorithm~\ref{alg:lazy_label_efficient} chooses $X_t$ at round $t$, suffers $\ell_t(X_t)$, and possibly chooses to query all the losses. In the worst case, we have $\ell_t(X_t)<\ell_t(r)=1$, for all $t$.

Let $M_T$ denote the number of switches. Using Lemma~\ref{LemRTLazy}, we may bound the policy regret:
\[
\bbE[R_T]\leq T\sqrt{\frac{2\ln N}{m}} + m + M_T.
\]
Note that
\begin{equation*}
M_T\leq\sum_{t=1}^TP(Z_t=1)=\epsilon T=m.
\end{equation*}
Setting \mbox{$m=T^{2/3}$,} we obtain an upper bound of $\cO(T^{2/3})$.
\end{proof}

\begin{algorithm}[]
    \caption{Lazy Revealing Action Algorithm}
    \begin{algorithmic}[1]
    \renewcommand{\algorithmicrequire}{\textbf{Input:}}
    \renewcommand{\algorithmicensure}{\textbf{Output:}}
    \REQUIRE Time horizon $T$, $0\leq\epsilon\leq 1$, learning rate $\eta>0$, feedback graph $G=(V, E)$ that characterizes a revealing action game, a revealing action $r$.
    \ENSURE Sequence of actions $X_{1:T}$
    \STATE Set $w_{1,0},\ldots,w_{|V|,0}=1$, $Z_0=1$
    \FOR {round $t=1,2\ldots$}
        \IF {$(Z_{t-1}=1)$}
            \STATE Draw action $J_t$ from $V$ according to the distribution $p_{i,t}=\frac{w_{i,t-1}}{\sum_{j\in V}w_{j,t-1}}$, $\forall i\in V$
        \ELSE 
            \STATE Set $J_t=X_{t-1}$
        \ENDIF
        \STATE Draw $Z_t\sim\text{Bernoulli}(\epsilon)$
        \IF {$(Z_t=1)$}
            \STATE Play the revealing action; i.e., $X_t=r$. Observe $\ell_t(i)$, $\forall i\in V$, and compute $w_{i,t}=w_{i, t-1}e^{-\eta\ell_t(i)/\epsilon}$.
        \ELSE 
            \STATE Play $X_t=J_t$, and set $w_{i,t}=w_{i,t-1}$ for each $i\in V$
        \ENDIF
    \ENDFOR
    \RETURN $X_{1:T}$
    \end{algorithmic} 
    \label{alg:lazy_revealing_action_alg}
\end{algorithm}

%% file: appendices/appendix_a.tex
\section{Visual summary of contributions}
\label{AppSummary}

\begin{table}[H]
  \renewcommand{\arraystretch}{1}
  \caption{Summary of upper and lower bounds of online learning problems
  Our contributions are highlighted in boldface.}
  \label{tab:sota_summary}
  \centering
  \begin{tabular}{|c||c|c|c|c|}
    \hline
      & oblivious & switching cost & memory of size 1 & bounded memory \\
    \hline\hline
    \multicolumn{5}{|c|}{Full-information feedback}\\
    \hline
    $\widetilde\cO$ & $\sqrt{T}$ & $\sqrt{T}$ & $T^{2/3}$ & $T^{2/3}$\\
    \hline
    $\Omega$ & $\sqrt{T}$ & $\sqrt{T}$ & $\sqrt{T}\rightarrow\bm{T^{2/3}}$ & $T^{2/3}$\\
    \hline \hline
    \multicolumn{5}{|c|}{Bandit feedback}\\
    \hline
    $\widetilde\cO$ & $\sqrt{T}$ & $T^{2/3}$ & $T^{2/3}$ & $T^{2/3}$\\
    \hline
    $\Omega$ & $\sqrt{T}$ & $T^{2/3}$ & $T^{2/3}$ & $T^{2/3}$\\
    \hline \hline
    \multicolumn{5}{|c|}{Non-revealing strongly-observable feedback graph}\\
    \hline
    $\widetilde\cO$ & $\sqrt{T}$ & $\bm{T^{2/3}}$ & $\bm{T^{2/3}}$ & $\bm{T^{2/3}}$\\
    \hline
    $\Omega$ & $\sqrt{T}$ & $\bm{T^{2/3}}$ & $\bm{T^{2/3}}$ & $\bm{T^{2/3}}$\\
    \hline\hline
    \multicolumn{5}{|c|}{Revealing strongly-observable feedback graph}\\
    \hline
    $\widetilde\cO$ & $\sqrt{T}$ & $\bm{T^{2/3}}$ & $\bm{T^{2/3}}$ & $\bm{T^{2/3}}$\\
    \hline
    $\Omega$ & $\sqrt{T}$ & $\sqrt{T}$ & $\bm{T^{2/3}}$ & $\bm{T^{2/3}}$\\
    \hline\hline
    \multicolumn{5}{|c|}{Revealing action feedback graph}\\
    \hline
    $\widetilde\cO$ & $T^{2/3}$ & $\bm{T^{2/3}}$ & $\bm{T^{3/4}}$ & $\bm{T^{3/4}}$\\
    \hline
    $\Omega$ & $T^{2/3}$ & $T^{2/3}$ & $T^{2/3}$ & $T^{2/3}$\\
    \hline
    \multicolumn{5}{|c|}{Weakly-observable feedback graph}\\
    \hline
    $\widetilde\cO$ & $T^{2/3}$ & $\bm{T^{3/4}}$ & $\bm{T^{3/4}}$ & $\bm{T^{3/4}}$\\
    \hline
    $\Omega$ & $T^{2/3}$ & $T^{2/3}$ & $T^{2/3}$ & $T^{2/3}$\\
    \hline
  \end{tabular}
\end{table}

\begin{figure}[H]
\centering
\subfloat[Difficulty of games for oblivious adversaries~\cite{bartok2014partial}.
]
{\includegraphics[width=0.6\linewidth]{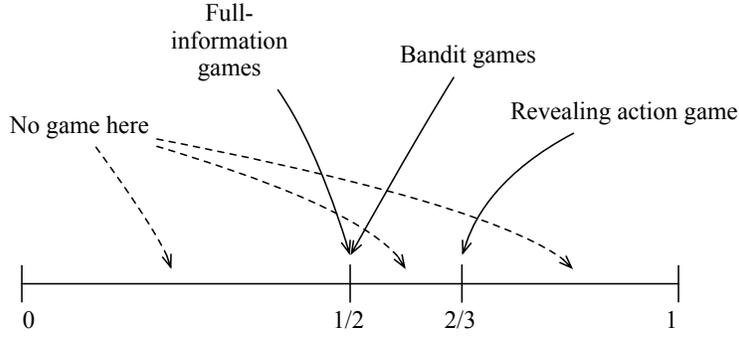}
\label{fig:example_games}}

\subfloat[Difficulty of games for feedback graphs and oblivious adversaries.]{\includegraphics[width=0.6\linewidth]{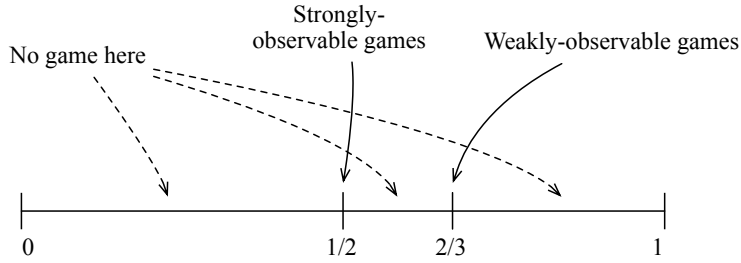}
\label{fig:graph_games}}

\subfloat[Difficulty of games for feedback graphs and adaptive adversaries. Our contributions are highlighted in boldface. An open question is whether games exist between $1/2$ and $2/3$, or between $2/3$ and $1$. We also show that Exp3.G achieves $\widetilde\cO(T^{3/4})$ regret for weakly-observable games. ]{\includegraphics[width=0.6\linewidth]{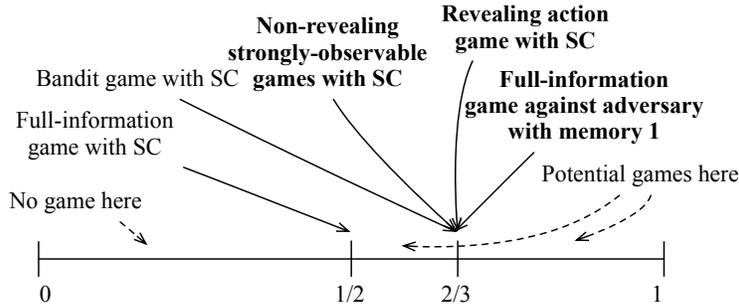}
\label{fig:adversarial_games}}

\caption{Diagrams showing the relative difficulty of games. The horizontal axis represents how the minimax standard/policy regret depends on $T$.}
\label{fig:game_difficulty}
\end{figure}

%% file: appendices/appendix_b.tex
\section{Forecasting algorithms}
\label{AppAlgorithms}
In this Appendix, we provide standard learning algorithms that we build upon in our analysis.

\begin{algorithm}[H]
    \caption{Exp3.G}
    \begin{algorithmic}[1]
    \renewcommand{\algorithmicrequire}{\textbf{Input:}}
    \renewcommand{\algorithmicensure}{\textbf{Output:}}
    \REQUIRE Feedback graph $G=(V,E)$, learning rate $\eta>0$, exploration set $U\subseteq V$, exploration rate $\gamma\in[0,1]$
    \ENSURE A sequence of actions $X_1, X_2,\ldots$
    \FOR {round $t=1,2\ldots$}
        \STATE Compute $p_t=(1-\gamma)q_t+\gamma u$, where $u$ is the uniform distribution on $U$
        \STATE Draw $X_t\sim p_t$, play $X_t$ and incur loss $\ell_t(X_t)$
        \STATE Observe $\{(i, \ell_t(i)):i\in \No(X_t)\}$
        \STATE Update:
            \begin{align*}
                \begin{split}    
                &\forall i\in V,\ \ \widehat \ell_t(i)=\frac{\ell_t(i)}{P_t(i)}\mathbbm{1}_{\{i\in \No(X_t)\}}, \\
                &\quad\quad\quad\quad \text{with}\ P_t(i)=\sum_{j\in \Ni(i)}p_t(j)
                \\
                &\forall i\in V,\ \ q_{t+1}(i)=\frac{q_t(i)\exp(-\eta\widehat \ell_t(i))}{\sum_{j\in V}q_t(j)\exp(-\eta\widehat \ell_t(j))}
                \end{split}
            \end{align*}
    \ENDFOR
    \RETURN $X_1, X_2,\ldots$
    \end{algorithmic} 
    \label{alg:exp3g}
\end{algorithm}

\begin{algorithm}[H]
    \caption{Lazy Forecaster for Label-Efficient Problem}
    \begin{algorithmic}[1]
    \renewcommand{\algorithmicrequire}{\textbf{Input:}}
    \renewcommand{\algorithmicensure}{\textbf{Output:}}
    \REQUIRE $0\leq\epsilon\leq 1$, learning rate $\eta>0$, actions $1, 2, \ldots, N$.
    \ENSURE Sequence of actions $X_1, X_2,\ldots$
    \STATE Set $w_{1, 0},\ldots, w_{N, 0}=1$, $Z_0=1$
    \FOR {round $t=1,2\ldots$}
        \IF {($Z_{t-1}=1)$}
            \STATE Draw action $X_t$ from $1,\ldots, N$ according to the distribution
            \[
                 \forall i=1,\ldots, N,\quad p_{i,t}=\frac{w_{i,t-1}}{\sum_{j=1,\ldots, N}w_{j,t-1}}
            \]
        \ELSE 
            \STATE Set $X_t=X_{t-1}$
        \ENDIF
        \STATE Draw $Z_t\sim\text{Bernoulli}(\epsilon)$
        \IF {$(Z_t=1)$}
            \STATE Choose to observe the losses $\ell_t(i)$, $\forall i=1,\ldots,N$, and compute
            \[
                 \forall i=1,\ldots, N,\quad w_{i,t}=w_{i, t-1}e^{-\eta\ell_t(i)/\epsilon}
            \]
        \ELSE 
            \STATE Set $w_{i,t}=w_{i,t-1}$ for each $i=1,\ldots, N$
        \ENDIF
    \ENDFOR
    \RETURN $X_1, X_2,\ldots$
    \end{algorithmic} 
    \label{alg:lazy_label_efficient}
\end{algorithm}

%% file: appendices/appendix_c.tex
\section{Supporting proofs for Section~\ref{sec:upper_bound}}
\label{AppUpper}

In this Appendix, we provide proofs of the auxiliary results involved in proving our upper bounds.

\subsection{Results from existing literature}
\label{AppUpperPrelim}

We begin by stating several lemmas appearing in the literature.

\begin{lemma}\label{thm:policyregretreduction} [Arora et al.~\cite{arora2012online}]
    Let $\mathcal{A}$ be an algorithm with standard
    regret against any sequence
    of $J$ loss functions generated by an adaptive
    adversary bounded by a monotonic function
    $R(J)$. Let $\mathcal{A}_\tau$ be the
    mini-batched version of $\mathcal{A}$ with batch size $\tau$. Let $(f_t)_{t\in[T]}$ be a sequence of loss functions generated by an adversary with memory of size $m$, let $X_1, \ldots, X_T$ be the sequence of
    actions played by $\mathcal{A}_\tau$ against this sequence, and let $y$
    be any action in the action space $\cX$ . If $\tau>m$, the policy regret of $\mathcal{A}_\tau$,
    compared to the constant action $y$, is bounded by
    \[
    \bbE\left[ \sum_{t=1}^Tf_t(X_{1,\ldots, t}) - f_t(\by^t) \right]\leq \tau R\left(\frac{T}{\tau}\right)+\frac{Tm}{\tau}+\tau.
    \]
Specifically, if $R(J)=CJ^q+o(J^q)$ for some $C>0$ and $q\in(0,1)$, and $\tau = C^{\frac{-1}{2-q}}T^{\frac{1-q}{2-q}}$, we have
    \[
    \bbE\left[ \sum_{t=1}^Tf_t(X_{1,\ldots, t}) - f_t(\by^t) \right]\leq C'T^{\frac{1}{2-q}}+o(T^{\frac{1}{2-q}}),
    \]
    where $C'=(m+1)C^{\frac{1}{2-q}}$.
\end{lemma}

\begin{lemma}\label{lma:orig_sec_order_bound} [Alon et al.~\cite{alon2015online}]
Consider a sequence of loss functions $\{\ell_t\}$ satisfying $\ell_t(i) \geq 0$ for all $t=1,\ldots, T$ and $i\in V$. Let $q_1, \ldots, q_T$ be the probability vectors defined as follows:
\[
q_t(i)=\frac{\exp{\left(-\eta\sum_{s=1}^{t-1}\ell_s(i)\right)}}{\sum_{j\in V}\exp{\left(-\eta\sum_{s=1}^{t-1}\ell_s(j)\right)}}, \quad \forall i\in V,
\]
where $\eta$ is a learning rate. For each $t$, let $S_t$ be a subset of $V$ such that $\ell_t(i)\leq \frac{1}{\eta}$ for all $i\in S_t$. Then for any $i^*\in V$, we have
\begin{align*}
\begin{split}
    &\sum_{t=1}^{T}\sum_{i\in V}q_t(i)\ell_t(i)-\sum_{t=1}^T\ell_t(i^*)\leq \frac{\ln K}{\eta} +\eta \sum_{t=1}^T\left( \sum_{i\in S_t}q_t(i)(1-q_t(i))\ell_t(i)^2 +\sum_{i\not\in S_t}q_t(i)\ell_t(i)^2 \right).
\end{split}
\end{align*}
\end{lemma}

\begin{lemma}\label{graphlma} [Alon et al.~\cite{alon2015online}]
    Let $G=(V,E)$ be a directed graph in which each node $i\in V$ is assigned a positive weight $w_i$. Assume that $\sum_{i\in V}w_i\leq 1$ and $\min_{i \in V} w_i\geq \epsilon$ for some constant $0<\epsilon<\frac{1}{2}$. Then
    \[
    \sum_{i\in V}\frac{w_i}{w_i+\sum_{j\in \Ni(i)}w_j}\leq 4\alpha\ln\frac{4|V|}{\alpha\epsilon},
    \]
    where $\alpha$ is the independence number of $G$.
\end{lemma}

\subsection{Proof of Theorem~\ref{thm:exp3gupperbound}}
\label{AppThmExp3}

The proof of Theorem~\ref{thm:exp3gupperbound} is based on the proof of Theorem 2 in Alon et al.~\cite{alon2015online}. In this proof, we will use a semicolon to separate fixed parameters from variables. For example, $\ell_t(i_{1:t-1}; i)$ denotes the loss function at the $t^{\text{th}}$ iteration, with $t-1$ fixed player's actions.

We first derive a lemma, a variant of Lemma~\ref{lma:orig_sec_order_bound}.

\begin{lemma}\label{modsecorderbound}
Fix any set $\{i_1, \ldots, i_{T-1}\}$ of player actions. Consider a sequence of loss functions $\{\ell_t\}$ satisfying $\ell_t(i_{1:t-1}; i) \geq 0$ for all $t=1,\ldots, T$ and $i\in V$, and let $\{q_t\}$ be defined by
    \begin{align}\label{eq:new_qt}
        q_t(i)=\frac{\exp{\left(-\eta\sum_{s=1}^{t-1}\ell_s(i_{1:s-1}; i)\right)}}{\sum_{j\in V}\exp{\left(-\eta\sum_{s=1}^{t-1}\ell_s(i_{1:s-1};j)\right)}}, \quad \forall i \in V,
    \end{align}
    for a learning rate $\eta$. For each $t$, let $S_t$ be a subset of $V$ such that $\ell_t(i_{1:t-1}; i)\leq \frac{1}{\eta}$ for all $i\in S_t$. Then for any $i^*\in V$, we have
\begin{align}
\label{EqnBeforeExp}
    \begin{split}
        &\sum_{t=1}^{T}\sum_{i\in V}q_t(i)\ell_t(i_{1:t-1};i)-\sum_{t=1}^T\ell_t(i_{1:t-1};i^*) \\
        & \qquad \leq\frac{\ln K}{\eta} +\eta \sum_{t=1}^T\Bigg( \sum_{i\in S_t}q_t(i)(1-q_t(i))\ell_t(i_{1:t-1};i)^2 +\sum_{i\not\in S_t}q_t(i)\ell_t(i_{1:t-1};i)^2 \Bigg).
    \end{split}
    \end{align}
\end{lemma}

\begin{proof}
For any round $1\leq t\leq T$, consider $\ell_t(i_{1:t-1}; i)$. Since $i_{1:t-1}$ is fixed, we can represent $\ell_t(i_{1:t-1}; i)$ using the function $\ell_t'(i)$, where $\ell_t'$ takes only one parameter. Note that $\ell_t'$ satisfies the condition in Lemma~\ref{lma:orig_sec_order_bound}, so the desired result follows as a direct corollary of Lemma~\ref{lma:orig_sec_order_bound}.
\end{proof}

We first consider the case where $G$ is strongly-observable and the exploration distribution is uniform on $V$.
Note that with respect to the fixed loss sequence $i_{1:t-1}$, the quantity $\widehat\ell_t(i_{1:t-1}; i)$ is an unbiased estimator of $\ell_t(i_{1:t-1}; i)$, for any $t$ and $i\in V$:
\begin{align}\label{eq:first_moment_bound}
	\mathop{\bbE}_{X_{t} \sim p_t}\left[ \widehat\ell_t(i_{1:t-1}; i) \mid X_{1:t-1}=i_{1:t-1} \right] = \ell_t(i_{1:t-1}; i).
\end{align}
For the second moment, we have
\begin{align}\label{eq:second_moment_bound}
	\mathop{\bbE}_{X_{t} \sim p_t}\left[ \widehat\ell_t(i_{1:t-1}; i)^2 \mid X_{1:t-1}=i_{1:t-1} \right] = \frac{\ell_t(i_{1:t-1}; i)^2}{P_t(i)}.
\end{align}

Assume WLOG that $K\geq 2$. Let $S=\{i:i\not\in \Ni(i)\}$; i.e., the set of nodes without self-loops. Then $V\backslash\{i\}=\Ni(i)$ for all $i \in S$, so $P_t(i)=1-p_t(i)$. Furthermore, 
\begin{equation*}
p_t(i) = (1-\gamma)q_t(i) + \frac{\gamma}{K} \leq1-\gamma + \frac{\gamma}{2} = 1-\eta.
\end{equation*}
Thus,
\begin{equation*}
\widehat{\ell}_t(i) \le \frac{\ell_t(i)}{P_t(i)} \le \frac{\ell_t(i)}{\eta} \le \frac{1}{\eta},
\end{equation*}
so we may apply Lemma~\ref{modsecorderbound}, with $S_t = S$ for all $t$, to the losses $\widehat\ell_1,\ldots,\widehat\ell_T$. Taking expectations of the bound~\eqref{EqnBeforeExp}, we have
\begin{align}\label{eq:upperbound_eq1}
\begin{split}
& \mathop{\bbE}_{X_{1:t-1}} \Bigg[ \sum_{t=1}^T\sum_{i\in V}q_t(i)\mathop{\bbE}_{X_{t}}\left[\widehat \ell_t(i_{1:t-1};i) \mid X_{1:t-1}\right] - \sum_{t=1}^T   \mathop{\bbE}_{X_{t}}\left[\widehat \ell_t(i_{1:t-1};i^*)|X_{1:t-1}\right] \Bigg] \\
        & \qquad \leq \frac{\ln K}{\eta} + \eta \sum_{t=1}^T \mathop{\bbE}_{X_{1:t-1}}\Bigg[ \sum_{i\in S}q_t(i)(1-q_t(i)) \mathop{\bbE}_{X_{t}}\left[\widehat \ell_t(i_{1:t-1}; i)^2 \mid X_{1:t-1}\right] \\
        & \qquad \qquad + \sum_{i\not\in S}q_t(i)\mathop{\bbE}_{X_{t}}[\widehat \ell_t(i_{1:t-1}; i)^2 \mid X_{1:t-1}] \Bigg] \\
& \qquad = \frac{\ln K}{\eta} + \eta \sum_{t=1}^T \mathop{\bbE}_{X_{1:t-1}}\Bigg[ \sum_{i\in S}q_t(i)(1-q_t(i)) \frac{\ell_t(i_{1:t-1}; i)^2}{P_t(i)} + \sum_{i\not\in S}q_t(i) \frac{\ell_t(i_{1:t-1}; i)^2}{P_t(i)}\Bigg].
    \end{split}
    \end{align}
Since $P_t(i)=1-p_t(i)$ and $\ell_t\leq 1$, the last term is further upper-bounded by
\begin{align*}
\eta \sum_{t=1}^T \mathop{\bbE}_{X_{1:t-1}}\left[ \sum_{i\in S}q_t(i)\frac{1-q_t(i)}{1-p_t(i)} + \sum_{i\not\in S}\frac{q_t(i)}{P_t(i)}  \right].
\end{align*}
Also note that
\[
\sum_{t=1}^T\sum_{i\in S}q_t(i)\frac{1-q_t(i)}{1-p_t(i)} \stackrel{(a)}{\leq} 2\sum_{t=1}^T\sum_{i\in S}q_t(i)\leq 2T,
\]
where inequality $(a)$ follows from the fact that if $q_t(i) \ge u$, then $p_t(i) \le q_t(i)$, so $\frac{1-q_t(i)}{1-p_t(i)} \le 1$; and if $q_t(i) \le u$, then $p_t(i) \le u \le \frac{1}{2}$, so $\frac{1-q_t(i)}{1-p_t(i)} \le \frac{1-q_t(i)}{1/2} \le 2$.

For any $i\not\in S$, we have $p_t(i)\geq\frac{\gamma}{K}$, since $i$ has a self-loop. Applying Lemma~\ref{graphlma} with $\epsilon=\frac{\gamma}{K}$ then yields
\[
\sum_{i\not\in S}\frac{q_t(i)}{P_t(i)}\leq 2\sum_{i\not\in S}\frac{p_t(i)}{P_t(i)}\leq8\alpha\ln\frac{4K^2}{\gamma},
\]
where the first inequality follows from the fact that
\begin{equation*}
p_t(i)\geq (1-\gamma)q_t(i)\geq\frac{1}{2}q_t(i).
\end{equation*}
Combining the bounds and using the inequality
\begin{equation*}
p_t(i)\leq q_t(i)+\gamma u(i),
\end{equation*}
we then obtain
\[
\sum_{i\in V}p_t(i)\ell_t(i_{1:t-1}; i)\leq\sum_{i\in V}q_t(i)\ell_t(i_{i:t-1};i)+\gamma,
\]
leading to the regret bound
\begin{align}
\begin{split}
    \mathop{\bbE}_{X_{1:t}}&\left[ \sum_{t=1}^T\sum_{i\in V}p_t(i)\ell_t(i_{1:t-1};i)\right] - \mathop{\bbE}_{X_{1:t}}\left[\sum_{t=1}^T \ell_t(i_{1:t-1};i^*) \right] \leq \gamma T+\frac{\ln K}{\eta}+2\eta T\left(1+4\alpha\ln\frac{4K^2}{\gamma}\right).
\end{split}
\end{align}
Substituting the chosen values of $\eta$ and $\gamma$ gives the desired result.

We now consider weakly-observable graphs $G$. Let $D\subseteq V$ be a weakly dominating set supporting the exploration distribution $u$, with $|D|=\delta$. Similarly to the proof of the strongly-observable case, we may apply Lemma~\ref{modsecorderbound} to the vectors $\widehat\ell_1,\ldots,\widehat\ell_T$, except we take $S_t=\emptyset$ for all $t$ in this case. This leads to the bound
\begin{align}\label{eq:std_strongly_regret}
\begin{split}
    \mathop{\bbE}_{X_{1:t}}&\left[ \sum_{t=1}^T\sum_{i\in V}p_t(i)\ell_t(i_{1:t-1};i)\right] - \mathop{\bbE}_{X_{1:t-1}}\left[\sum_{t=1}^T \ell_t(i_{1:t-1};i^*) \right] \leq \gamma T+\frac{\ln K}{\eta}+\eta \sum_{t=1}^T\mathop{\bbE}_{X_{1:t}}\left[\sum_{i\in V}\frac{q_t(i)}{P_t(i)}\right],
\end{split}
\end{align}
for any fixed $i^*\in V$.

To bound the expectation on the right-hand side, consider the set $S=\{i:i\not\in \Ni(i)\}$ of nodes without self-loops. Then
\begin{equation*}
P_t(i)=\sum_{j\in \Ni(i)}p_t(j)\geq\frac{\gamma}{\delta}, \quad \forall i \in S,
\end{equation*}
since if $i$ is weakly-observable, there exists $k\in D$ such that $k\in \Ni(i)$ and $p_t(k)\geq\frac{\gamma}{\delta}$, because the exploration distribution is uniform over $D$; if $i$ is strongly-observable, then $i$ must be dominated by all other nodes in $V$, including every node in $D$, so the same statement holds. For the vertices with self-loops, we use the bound
\begin{equation*}
P_t(i)\geq p_t(i)\geq(1-\gamma)q_t(i)\geq\frac{1}{2}q_t(i).
\end{equation*}
Hence,
\[
\sum_{i\in V}\frac{q_t(i)}{P_t(i)}=\sum_{i\in S}\frac{q_t(i)}{P_t(i)}+\sum_{i\not\in S}\frac{q_t(i)}{P_t(i)}\leq \frac{\delta}{\gamma}+2K.
\]
Putting everything together yields:
\begin{align}\label{eq:std_weakly_regret}
\begin{split}
    \mathop{\bbE}_{X_{1:t-1}}&\left[ \sum_{t=1}^T\sum_{i\in V}p_t(i)\ell_t(i_{1:t-1};i)\right] - \mathop{\bbE}_{X_{1:t-1}}\left[\sum_{t=1}^T \ell_t(i_{1:t-1};i^*) \right] \leq \gamma T+\frac{\ln K}{\eta}+\frac{\eta\delta}{\gamma}T+2\eta KT.
\end{split}
\end{align}
Using the chosen values of $\eta$ and $\gamma$ gives the desired result.


%% file: appendices/appendix_d.tex
\section{Supporting proofs for Section~\ref{SecLower}}
\label{AppLower}

In this Appendix, we provide proofs of the supporting results used to derive our lower bounds.

\subsection{Multi-scale random walk (MRW) construction}
\label{AppMRW}

\begin{algorithm}[H]
    \caption{Oblivious Loss Sequence Generating via Multi-scale Random Walk (MRW)}
    \begin{algorithmic}[1]
    \renewcommand{\algorithmicrequire}{\textbf{Input:}}
    \renewcommand{\algorithmicensure}{\textbf{Output:}}
    \REQUIRE Time horizon $T>0$, experts $x_1$ and $x_2$
    \ENSURE A sequence of oblivious losses $L_{1:T}$
    \STATE Set $\epsilon=2^{1/3}T^{-1/3}/(9\log_2T)$ and $\sigma=1/(9\log_2T)$
    \STATE Choose $Z\sim \text{Uniform}(\{-1,1\})$
    \STATE Draw $T$ i.i.d.\ Gaussians $\xi_1,\ldots,\xi_T\sim\mathcal{N}(0, \sigma^2)$
    \FOR {$t=0,\ldots, T$}
        \IF {($t=0$)}
            \STATE $W_0=0$
        \ELSE 
            \STATE $W_t=W_{\rho(t)}+\xi_t$, where $\rho(t)=t-2^{\delta(t)}$, $\delta(t)=\max\{i\geq 0: t \text{ mod } 2^i=0\}$
            \STATE $L'_t(x_1)=W_t+\frac{1}{2}$
            \STATE $L'_t(x_2)=W_t+\frac{1}{2}+Z\cdot\epsilon$
            \STATE For $i=1,2: L_t(x_i)=\text{clip}(L'_t(x_i))$, where $\text{clip}(\alpha)=\min\{\max\{\alpha, 0\}, 1\}$
        \ENDIF
    \ENDFOR
    \RETURN $L_{1:T}$ 
    \end{algorithmic} 
    \label{alg:loss_seq_generation}
\end{algorithm}


\subsection{Proof of Theorem~\ref{thm:full_info_lower_bound}}
\label{AppThmFullInfo}



By Yao's minimax principle~\cite{Yao77}, it suffices to construct a stochastic loss sequence such that the expected regret is bounded below against the optimal deterministic player. Accordingly, we define the loss functions as
\begin{equation*}
f_t(X_{t-1}, X_t) = 1\{X_{t-1} \neq X_t\} + L_{t-1}(X_{t-1}),
\end{equation*}
where $L_0 = 0$ and $1\{X_0 \neq X_1\}  = 0$, and the loss sequence $\{L_t\}$ is defined in Algorithm~\ref{alg:loss_seq_generation}.

Let $M_{T+1}=\sum_{t=1}^{T+1}\mathbbm{1}_{\{X_{t-1}\neq X_t\}}$ be the number of switches, where $\mathbbm{1}_{\{X_0\neq X_1\}}=0$. We rewrite $R_{T+1}$ as follows:
\begin{align*}
\begin{split}
    R_{T+1}=&\sum_{t=1}^{T+1}f_t(X_{t-1}, X_t)-\min_{x\in\{x_1, x_2\}}\sum_{t=1}^{T+1}f_t(x, x)\\
    =&\sum_{t=2}^{T+1}L_{t-1}(X_{t-1})+M_{T+1}-\min_{x\in\{x_1, x_2\}}\sum_{t=2}^{T+1}L_{t-1}(x)\\
    =&\sum_{t=1}^{T}L_{t}(X_{t})+M_{T+1}-\min_{x\in\{x_1, x_2\}}\sum_{t=1}^TL_{t}(x).
\end{split}
\end{align*}
We also define the auxiliary regret term
\begin{equation*}
    R_{T+1}' = \sum_{t=1}^{T}L'_{t}(X_{t})+M_{T+1}-\min_{x\in\{x_1, x_2\}}\sum_{t=}^TL'_{t}(x),
\end{equation*}
defined with respect to the unclipped loss sequence in Algorithm~\ref{alg:loss_seq_generation}.

As derived by Dekel et al.~\cite{dekel2014bandits}, the bounded loss sequence $\{L_t\}$ satisfies $d(\rho), w(\rho) \le \lfloor \log_2 T\rfloor + 1$, where the width $w$ is defined by
\begin{align*}
\text{cut}(t) &= \{s\in[T]: \rho(s)<t\leq s\}, \\
w(\rho) &= \max_{t\in[T]}|\text{cut}(t)|,
\end{align*}
and the depth is defined by $d(\rho)=\max_{t\in[T]}|\rho^*(t)|$, where the ancestor function $\rho^*$ is defined recursively:
\begin{align*}
\begin{split}
    &\rho^*(0)=\{\},\\
    \forall t\in[T]\ &\rho^*(t)=\rho^*\left(\rho(t)\right)\cup \{\rho(t)\}.
\end{split}
\end{align*}
Furthermore,
\begin{equation}
\label{EqnMRW}
\forall\delta\in(0,1),\ P\left(\max_{t\in[T]}|W_t|\leq\sigma\sqrt{2d(\rho)\log(T/\delta)}\right)\geq 1-\delta.
\end{equation}

Let $Y_t$ denote the unclipped losses at round $t$. In our case, $Y_1=\{0,0\}$ and
\begin{equation*}
Y_t=\left\{L'_{t-1}(X_{t-1}),\ L'_{t-1}(X_{t-1})\right\}, \quad \text{for } 2 \le t \le T+1.
\end{equation*}
Also, let $Z_t$ be the clipped losses at round $t$. In our case, $Z_1=\{0,0\}$ and
\begin{equation*}
Z_t=\left\{L_{t-1}(X_{t-1}),\ L_{t-1}(X_{t-1})\right\}, \quad \text{for } 2 \le t \le T+1.
\end{equation*}

Let $\cF$ denote the $\sigma$-algebra generated by $Z_{1:T+1}$, and let $\cF'$ denote the $\sigma$-algebra generated by $Y_{1:T+1}$. 
Let $\mathbb{S}=\bbP(\cdot|Z>0)$ and $\bbQ=\bbP(\cdot|Z<0)$ denote the conditional probabilities.
We then have the following lemma:

\begin{lemma}\label{kldivswitches}
For any event $A\in\cF$,
\[
    |\bbS(A)-\bbQ(A)|\leq\frac{\epsilon}{\sigma}\sqrt{w(\rho)\bbE_{\P}[M_{T+1}]},
\]
\end{lemma}

\begin{proof}
Note that $\cF \subseteq \cF'$; we will derive an upper bound on $\sup_{A \in \cF'} |\bbS(A) - \bbQ(A)|$.

We use the chain rule for relative entropy:
\[
\kl(\S||\Q)=\sum_{t=1}^{T+1}\kl(\S_{t-1}||\Q_{t-1})
\]
where
\begin{align*}
\S_{t-1} = \S(\cdot|Y_1,\ldots,Y_{t-1}), \qquad \text{and} \qquad \Q_{t-1} = \Q(\cdot|Y_1,\ldots,Y_{t-1}).
\end{align*}
We have the following three cases:
\begin{enumerate}
    \item If $X_{t-1}=X_{\rho({t-1})}$, then under both $\S_{t-1}$ and $\Q_{t-1}$, we observe $\{L'_{\rho(t-1)}+\xi_{t-1},\  L'_{\rho(t-1)}+\xi_{t-1}\}$.
    \item 
    If $X_{t-1}=x_1$ and $X_{\rho({t-1})}=x_2$, we observe $\{L'_{\rho(t-1)}(x_2)+\xi_{t-1}-\epsilon,\  L'_{\rho(t-1)}(x_2)+\xi_{t-1}-\epsilon\}$ under $\S_{t-1}$, but $\{L'_{\rho(t-1)}(x_2)+\xi_{t-1}+\epsilon,\  L'_{\rho(t-1)}(x_2)+\xi_{t-1}+\epsilon\}$ under $\Q_{t-1}$.
    \item 
    If $X_{t-1}=x_2$ and $X_{\rho({t-1})}=x_1$, we observe $\{L'_{\rho(t-1)}(x_1)+\xi_{t-1}+\epsilon,\  L'_{\rho(t-1)}(x_1)+\xi_{t-1}+\epsilon\}$ under $\S_{t-1}$, but $\{L'_{\rho(t-1)}(x_1)+\xi_{t-1}-\epsilon,\  L'_{\rho(t-1)}(x_1)+\xi_{t-1}-\epsilon\}$ under $\Q_{t-1}$.
\end{enumerate}
In both cases 2 and 3, we observe two normal distributions of which the means differ by $2\epsilon$. Hence,
\begin{align*}
\begin{split}
    \kl(\S||\Q)&\leq\frac{2\epsilon^2}{\sigma^2}\bbE\left[ \sum_{t=1}^{T+1}\mathbbm{1}_{\{X_{t-1}\neq X_{\rho(t-1)}\}}\Bigg|Z>0 \right]\\
    &\leq\frac{2\epsilon^2}{\sigma^2}\bbE\left[ \sum_{t=1}^{T+1}\mathbbm{1}_{\{X_{t}\neq X_{\rho(t)}\}}\Bigg|Z>0 \right].
\end{split}
\end{align*}
Notice that the event $X_{t}\neq X_{\rho(t)}$ occurs if there is at least one time $s$ of switching such that $t\in\text{cut}(s)$. Let $S_{1:M_{T+1}}$ denote the random sequence of times of such switches. Then
\begin{align*}
\sum_{t=1}^{T+1}\mathbbm{1}_{\{X_{t}\neq X_{\rho(t)}\}} &\leq \sum_{r=1}^{M_{T+1}}\sum_{t\in\text{cut}(S_r)}\mathbbm{1}_{\{X_{t}\neq X_{\rho(t)}\}} \leq w(\rho)M_{T+1},
\end{align*}
implying that
\[
\kl(\S||\Q)\leq\frac{2\epsilon^2}{\sigma^2}\bbE\left[ w(\rho)M_{T+1}|Z>0 \right].
\]
Using a nearly identical argument, we may obtain an upper bound with the conditioning on $Z < 0$.
Since $\P=\frac{1}{2}(\S+\Q)$, we arrive at the inequality
\[
\kl(\S||\Q)\leq\frac{2\epsilon^2}{\sigma^2}\bbE_\P[ w(\rho)M_{T+1}].
\]
Applying Pinsker's inequality, we then have
\[
	\sup_{A\in\cF'} |\bbS(A)-\bbQ(A)| \leq\frac{\epsilon}{\sigma}\sqrt{w(\rho)\bbE_{\P}[M_{T+1}]},
\]
as wanted.
\end{proof}

The proof then proceeds by showing that \mbox{$\bbE[R_{T+1}']=\widetilde\Omega(T^{2/3})$}, as stated in the following lemma:
\begin{lemma}\label{lma:pseudo_regret_lower_bound}
We have the lower bound 
\begin{align*}
\bbE[R_{T+1}']\geq \frac{2T^{2/3}}{125\log_2T}=\widetilde\Omega(T^{2/3}).
\end{align*}
\end{lemma}

\begin{proof}
Let $A$ be the event that the worse action ($x_2$ if $Z>0$, and $x_1$ if $Z<0$) is chosen at least $(T+1)/2$ times. Then
\begin{align}
\begin{split}
    \bbE[R_{T+1}']&\geq\bbE\left[ \max\left\{M_{T+1}, \frac{\epsilon T}{2}\mathbbm{1}_{\{A\}}\right\} \right] \geq \bbE\left[\frac{1}{2}\left(M_{T+1}+\frac{\epsilon T}{2}\mathbbm{1}_{\{A\}}\right)\right] =\frac{1}{2}\bbE[M_{T+1}]+\frac{\epsilon T}{4}\P(A).
\end{split}
\end{align}

Let $A_1$ be the event that $x_1$ is chosen at least $\frac{T+1}{2}$ times, and let $A_2$ be the event that $x_2$ is chosen at least $\frac{T+1}{2}$ times. Using the fact that $\P(A)=\frac{1}{2}(\S(A_2)+\Q(A_1))$, we have
\[
\bbE[R_{T+1}']\geq\frac{1}{2}\bbE[M_{T+1}]+\frac{\epsilon T}{8}(\S(A_2)+\Q(A_1)).
\]
Applying Lemma~\ref{kldivswitches} then gives
\begin{align*}
\begin{split}
    \bbE[R_{T+1}']&\geq\frac{1}{2}\bbE[M_{T+1}] +\frac{\epsilon T}{8}\left(\S(A_2)+\S(A_1)-\frac{\epsilon}{\sigma}\sqrt{w(\rho)\bbE[M_{T+1}]}\right)\\
    &\geq \frac{1}{2}\bbE[M_{T+1}] +\frac{\epsilon T}{8}\left(\S(A_1\cup A_2)-\frac{\epsilon}{\sigma}\sqrt{w(\rho)\bbE[M_{T+1}]}\right)\\
    & = \frac{1}{2}\bbE[M_{T+1}]+\frac{\epsilon T}{8}\left(1-\frac{\epsilon}{\sigma}\sqrt{w(\rho)\bbE[M_{T+1}]}\right)\\
    & = \frac{1}{2}\bbE[M_{T+1}]+\frac{\epsilon T}{8}-\frac{\epsilon^2T}{8\sigma}\sqrt{w(\rho)\bbE[M_{T+1}]}.
\end{split}
\end{align*}
The last expression is quadratic in $\sqrt{\bbE[M_{T+1}]}$, so it is bounded below by
\[
\frac{\epsilon T}{8}-\frac{\epsilon^4T^2}{128\sigma^2}w(\rho).
\]
Recall that $\epsilon = \frac{2^{1/3}T^{-1/3}}{9\log_2T}$, $\sigma = \frac{1}{9\log_2T}$, and $w(\rho) \le \lc{\log_2T}\rc+1\leq2\log_2T$. Substituting these values into the above inequality, we obtain
\begin{align}\label{eq:pseudoregretbound}
\bbE[R_{T+1}']\geq \frac{2T^{2/3}}{125\log_2T}=\widetilde\Omega(T^{2/3}).
\end{align}
\end{proof}

Next, we show that $\bbE[R_{T+1}]$ is close to $\bbE[R_{T+1}']$:
\begin{lemma}\label{lma:bound_regret_and_pseudo_regret}
When $T \ge 10$, we have
\begin{equation*}
\bbE[R_{T+1}]\geq \bbE[R_{T+1}']-\frac{\epsilon T}{10}.
\end{equation*}
\end{lemma}

\begin{proof}
Consider the event $B=\{\forall t\in[T]: L_t=L'_t\}$. We know that $W_{1:T}$ has depth $d\leq\lc{\log_2T}\rc+1\leq 2\log_2T$. Using inequality~\eqref{EqnMRW} with $\delta=\frac{1}{T} \leq \frac{1}{10}$, with probability at least $\frac{9}{10}$, we have
\begin{align*}
\max_{1 \le t \le T} |W_t| & \leq \sigma\sqrt{2d\log\frac{T}{\delta}} \leq\sigma\sqrt{10\log_2T\log T}  \leq 4\sigma\log_2T.
\end{align*}
Hence, setting $\sigma=\frac{1}{9\log_2T}$ yields
\[
P\left(\forall t\in[T],\ \frac{1}{2}+W_t\in\left[\frac{1}{18},\frac{17}{18}\right]\right)\geq\frac{9}{10}.
\]
Furthermore, $\epsilon < \frac{1}{18}$, so $L_t'(x_1), L_t'(x_2)\in[0,1]$ whenever $\frac{1}{2}+W_t\in\left[\frac{1}{18},\frac{17}{18}\right]$. This implies that $P(B)\geq 9/10$. If $B$ happens, then $R_{T+1}=R_{T+1}'$; otherwise,
\begin{equation*}
M_{T+1}\leq R_{T+1}\leq R'_{T+1}\leq M_{T+1}+\epsilon T,
\end{equation*}
so $R_{T+1}'-R_{T+1}\leq\epsilon T$. Therefore,
\begin{align*}
\bbE[R_{T+1}']-\bbE[R_{T+1}] &= \bbE[R_{T+1}'-R_{T+1}|\neg B]\cdot P(\neg B) \leq \frac{\epsilon T}{10}.
\end{align*}
\end{proof}

Combining Lemma~\ref{lma:pseudo_regret_lower_bound} and Lemma~\ref{lma:bound_regret_and_pseudo_regret}, we obtain Theorem~\ref{thm:full_info_lower_bound}.


\subsection{Supporting lemmas for Theorem~\ref{ThmSC}}
\label{AppSC}

In this subsection, we provide statements and proofs of supporting lemmas for Theorem~\ref{ThmSC}. We begin with a simple lemma:


\begin{lemma}\label{lem:ind_num}
If the independence number of a strongly-observable graph $G=(V,E)$ is $1$, then $G$ is revealing.
\end{lemma}

\begin{proof}
Suppose $G$ is non-revealing. Then there exist vertices $u,v \in V$ such that $\Ni(u)\cap \Ni(v)=\emptyset$. If $u$ and $v$ were not connected by an edge, then $\alpha(G) \ge 2$, a contradiction. Suppose there were an edge between $u$ and $v$, and assume WLOG that $u\in\Ni(v)$. If $u$ has a self-loop, then $\{u\}\subseteq \Ni(u)\cap\Ni(v)$. Otherwise, by the property of strong observability, we have $V\backslash\{u\}\subseteq\Ni(u)$, implying that $v\in\Ni(u)$. We then have the following scenarios:
\begin{itemize}
\item If $v$ has a self-loop, then $\{v\}\subseteq \Ni(u)\cap\Ni(v)$.
\item If $v$ does not have a self-loop, then $V\backslash\{v\}\subseteq\Ni(v)$ and $V\backslash\{v, u\}\subseteq\Ni(u)\cap\Ni(v)$.
\end{itemize}
In both cases, $\Ni(u)\cap \Ni(v) \neq \emptyset$, leading to a contradiction.
\end{proof}

\begin{figure}[]
\centering
\subfloat[]{\includegraphics[width=0.2\linewidth]{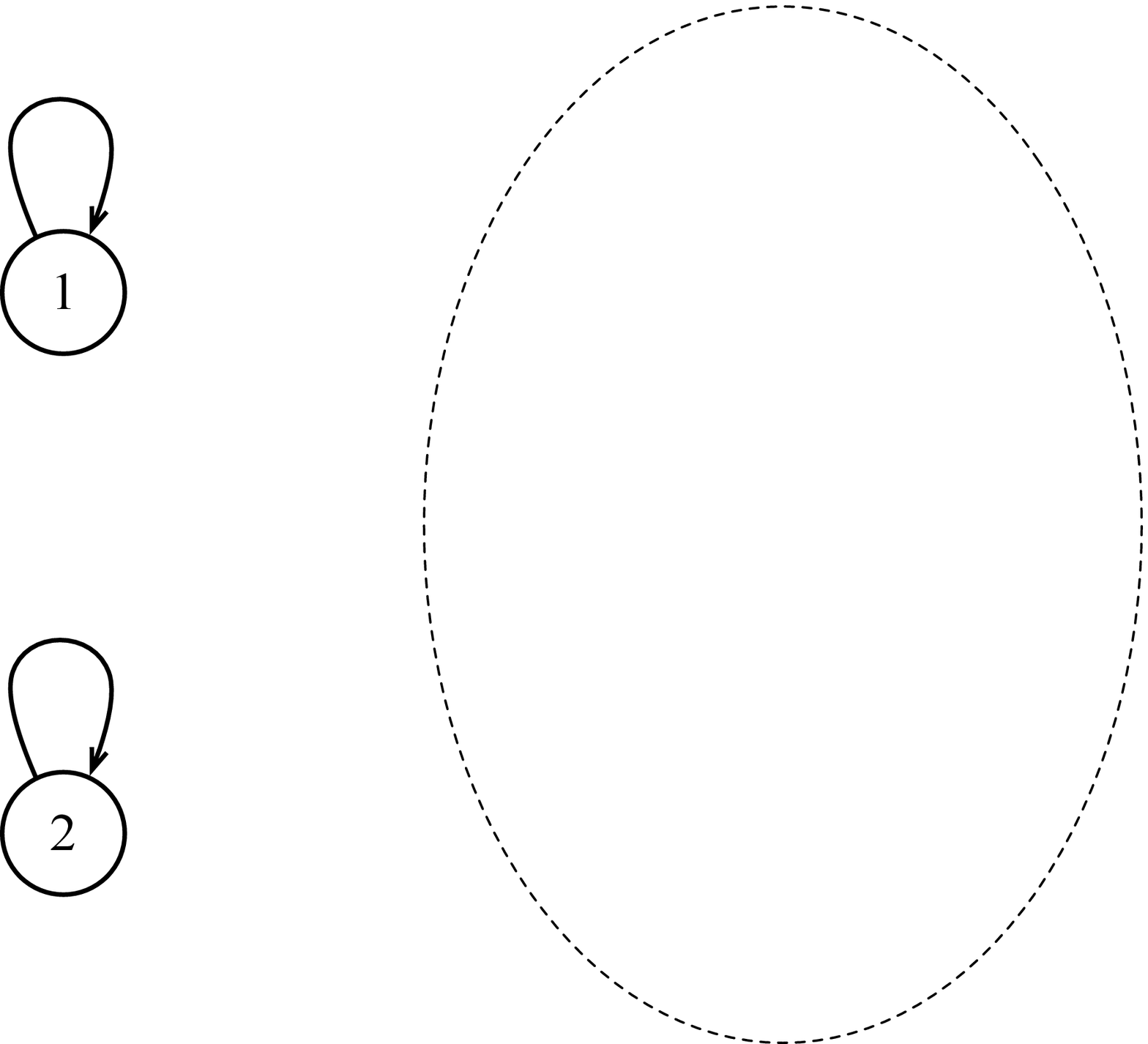}
\label{fig:taba}}\hfil
\subfloat[]{\includegraphics[width=0.2\linewidth]{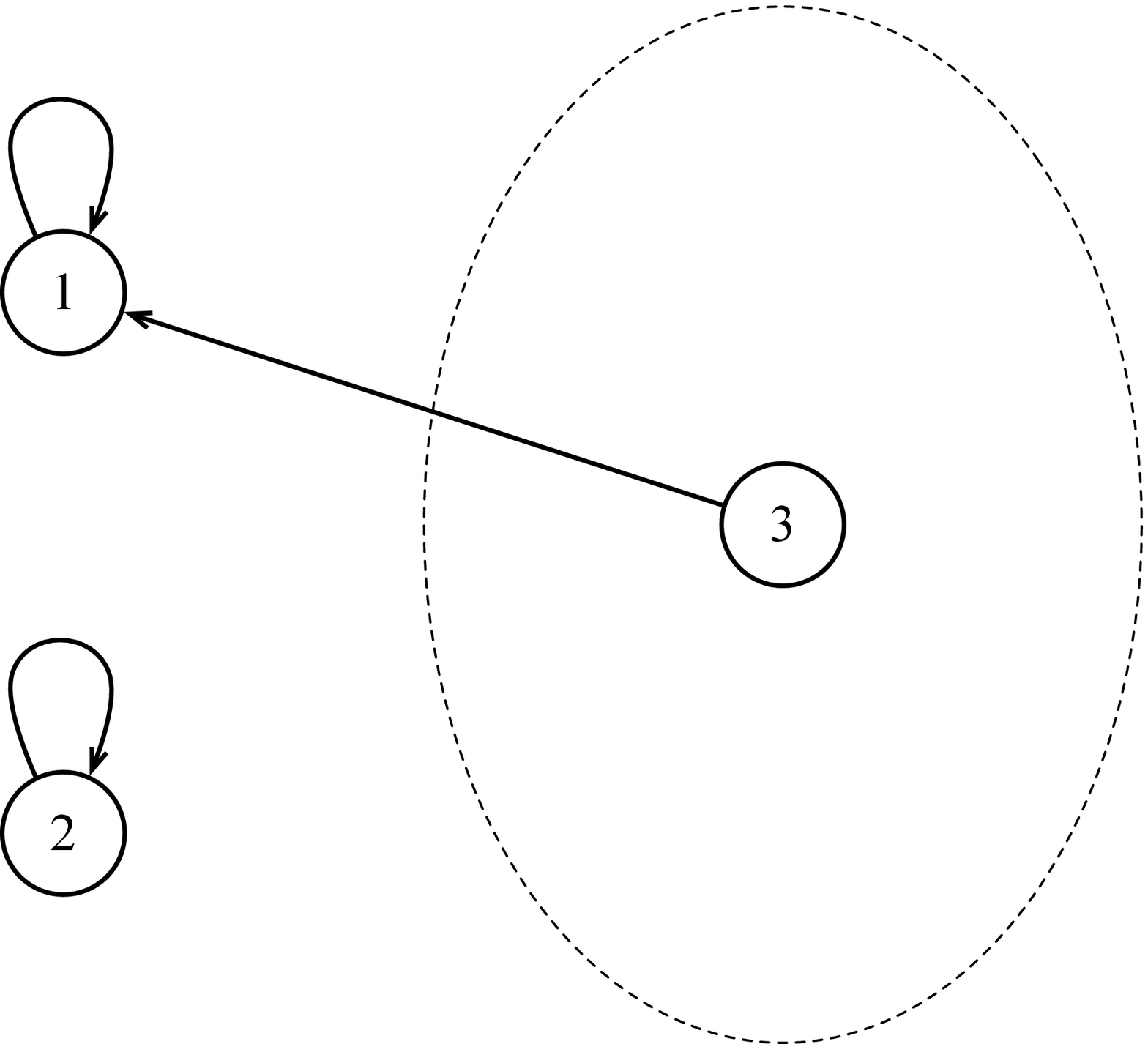}
\label{fig:tabb}}

\subfloat[]{\includegraphics[width=0.2\linewidth]{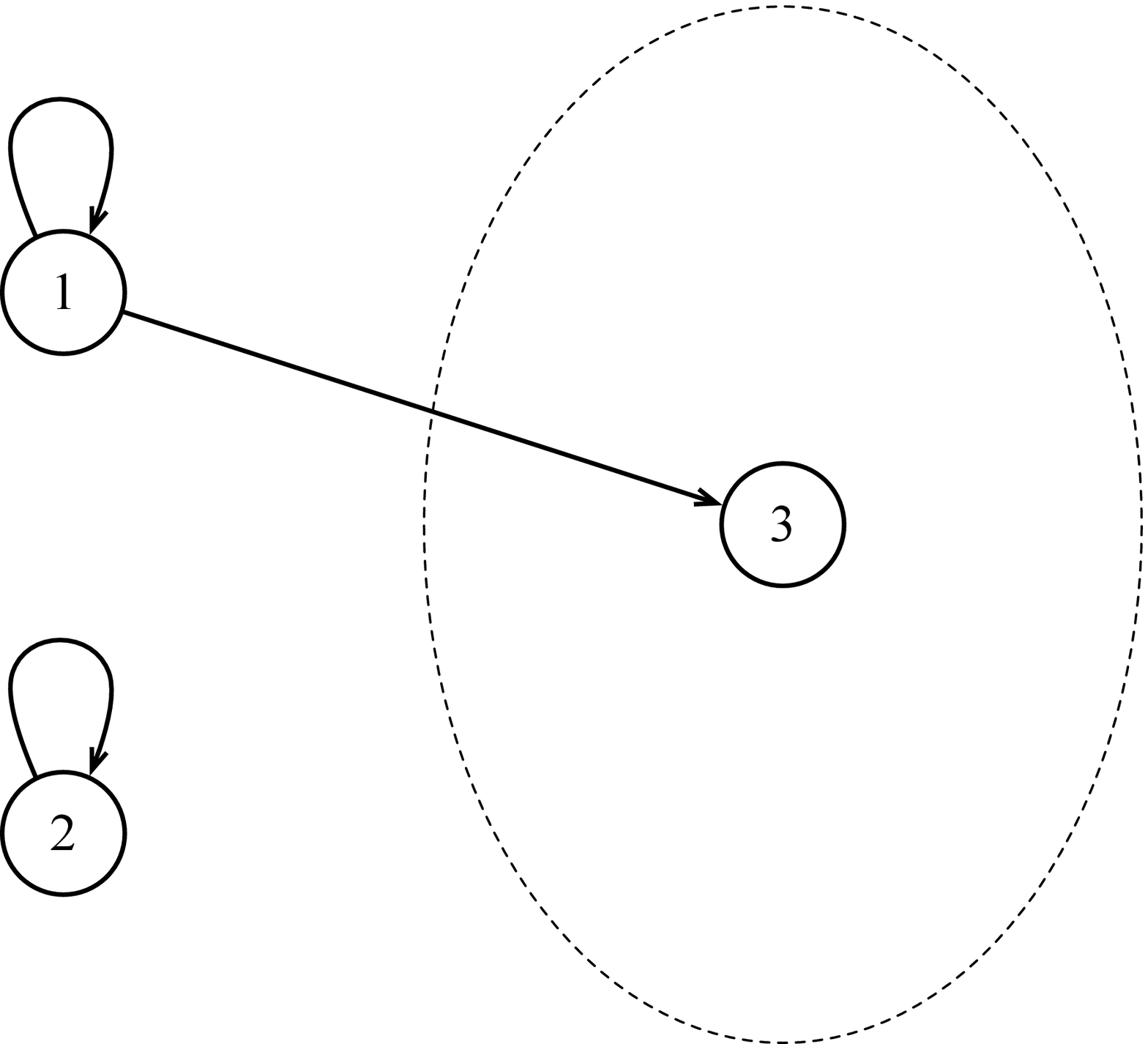}
\label{fig:tabc}}\hfil
\subfloat[]{\includegraphics[width=0.2\linewidth]{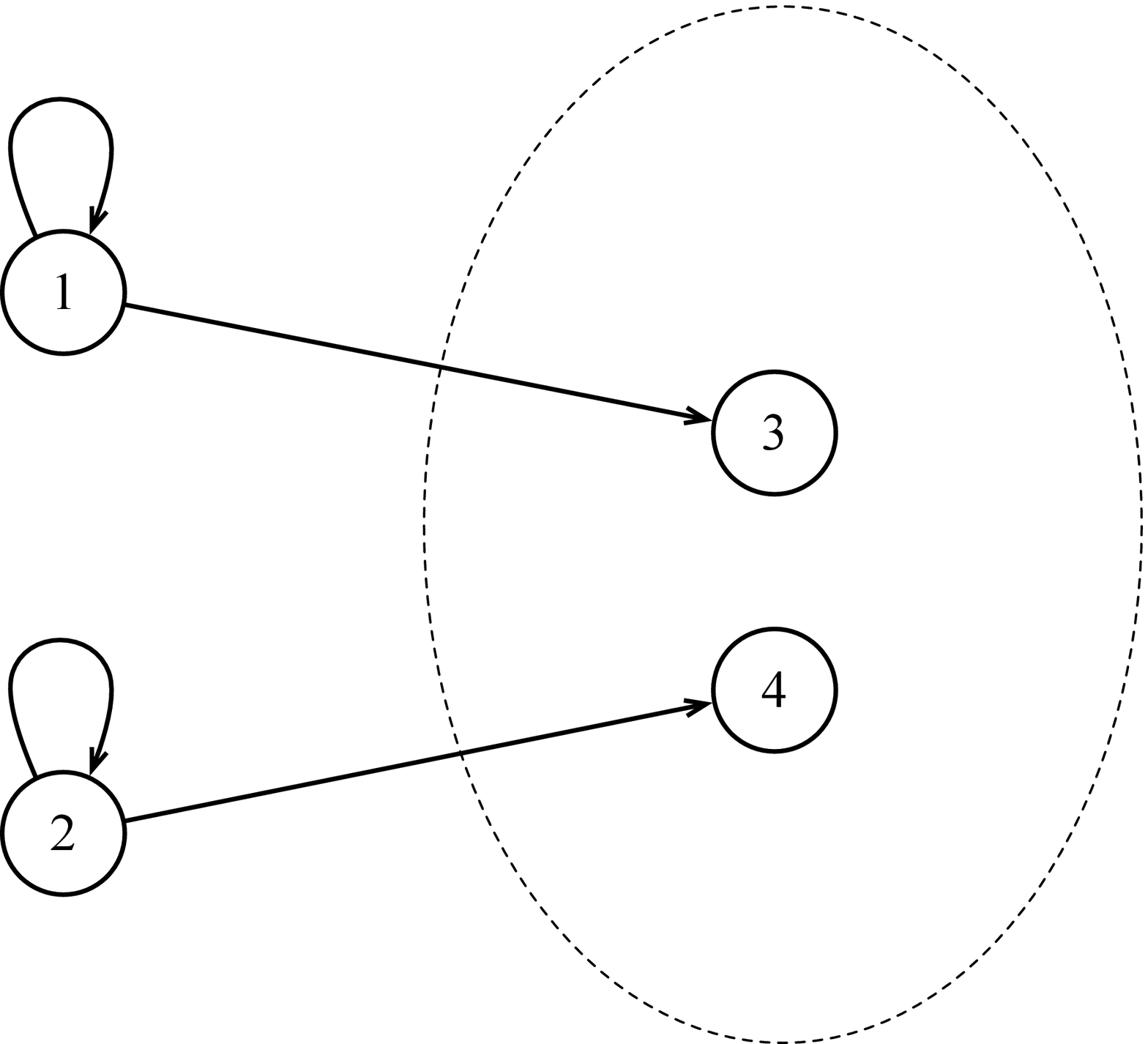}
\label{fig:tabd}}

\subfloat[]{\includegraphics[width=0.2\linewidth]{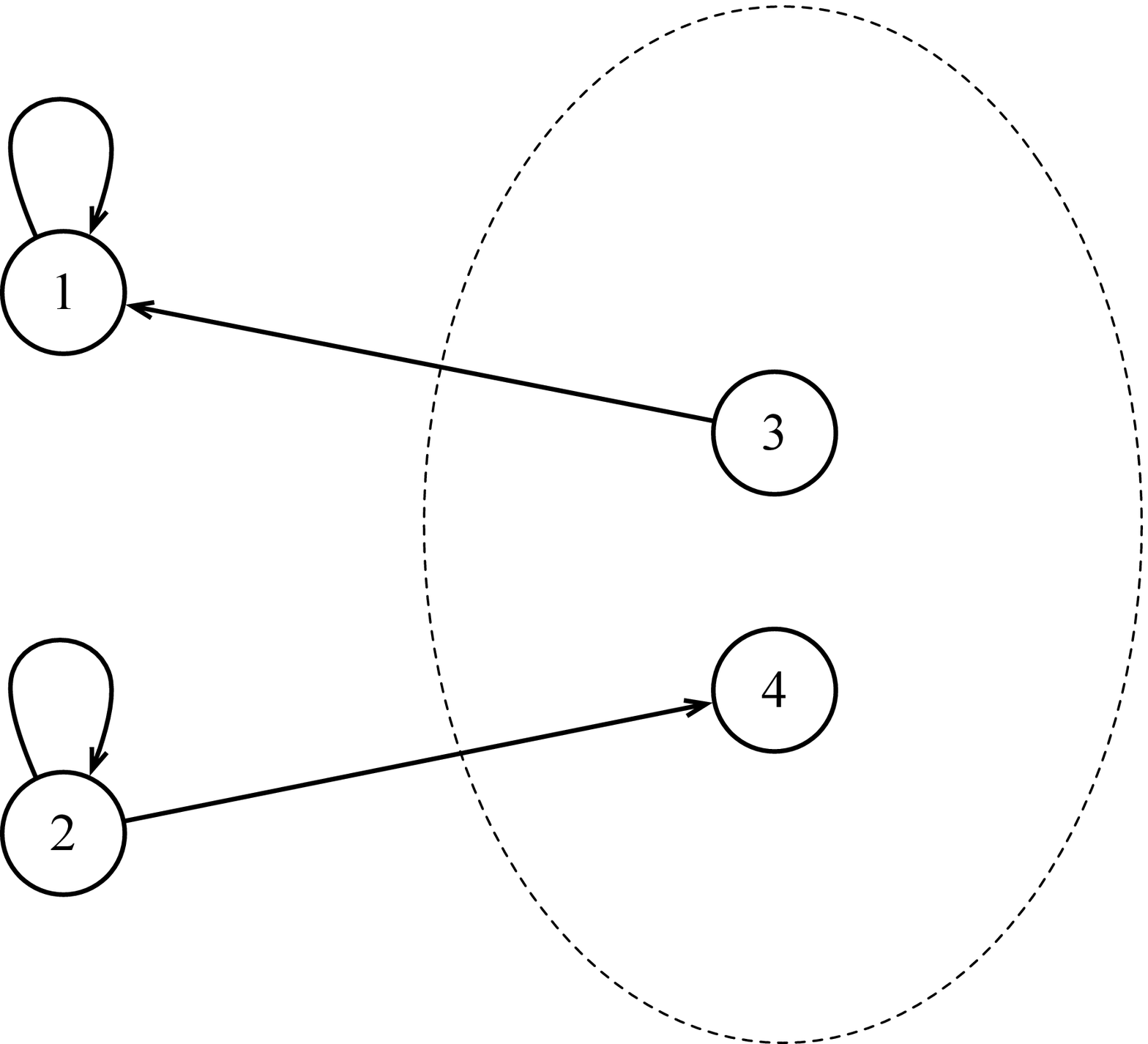}
\label{fig:tabe}}\hfil
\subfloat[]{\includegraphics[width=0.2\linewidth]{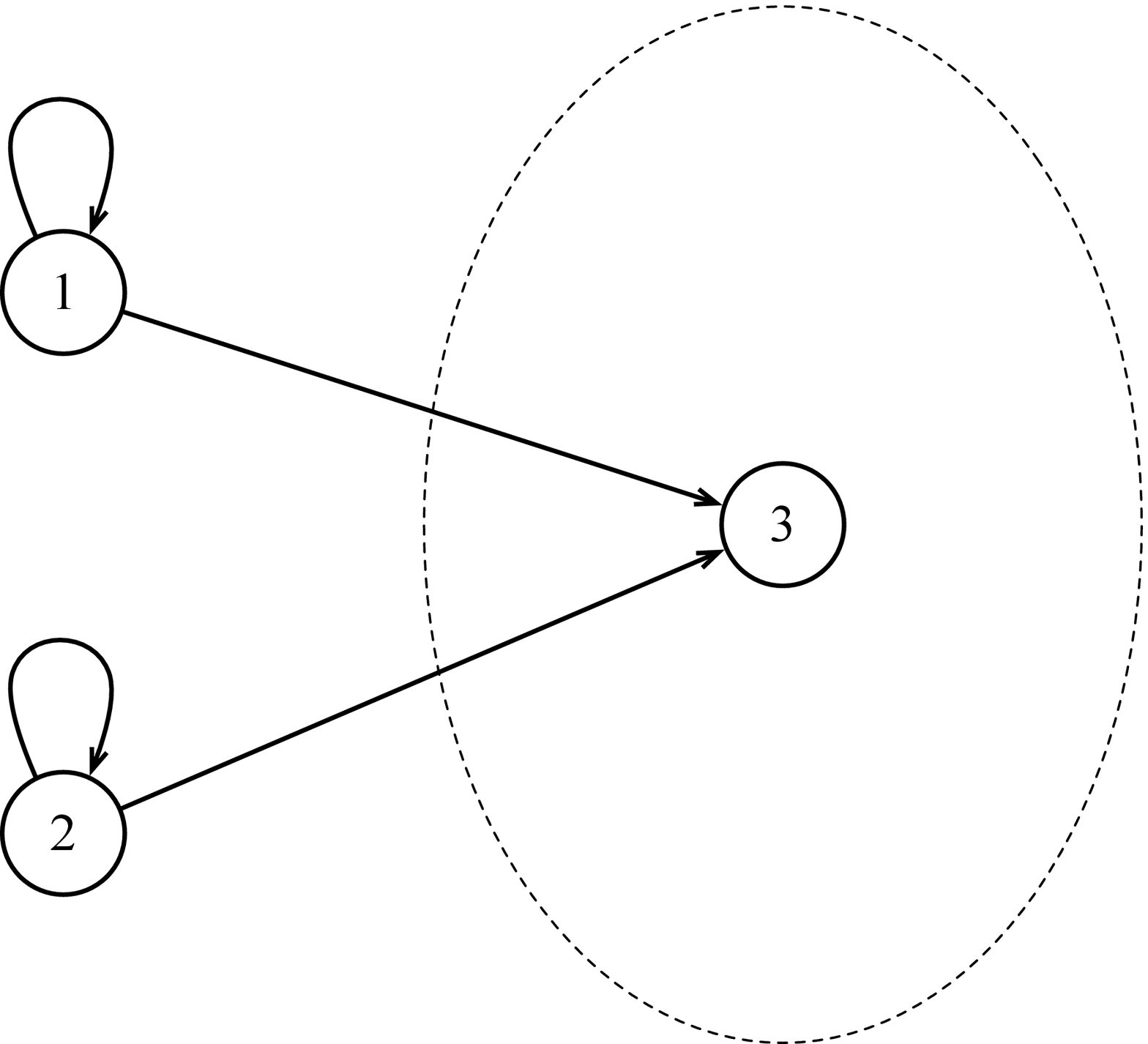}
\label{fig:tabf}}

\subfloat[]{\includegraphics[width=0.2\linewidth]{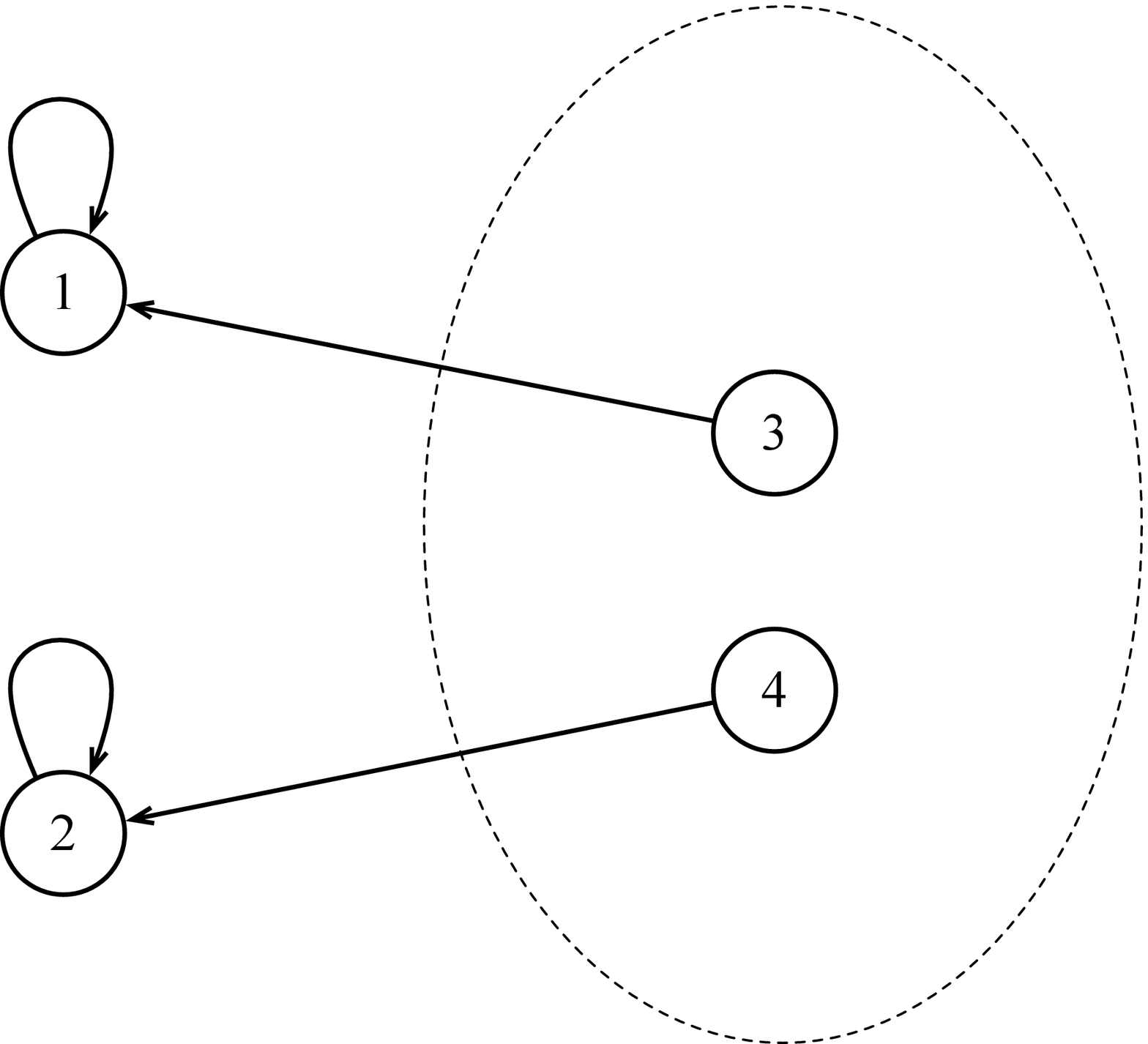}
\label{fig:tabg}}\hfil
\subfloat[]{\includegraphics[width=0.2\linewidth]{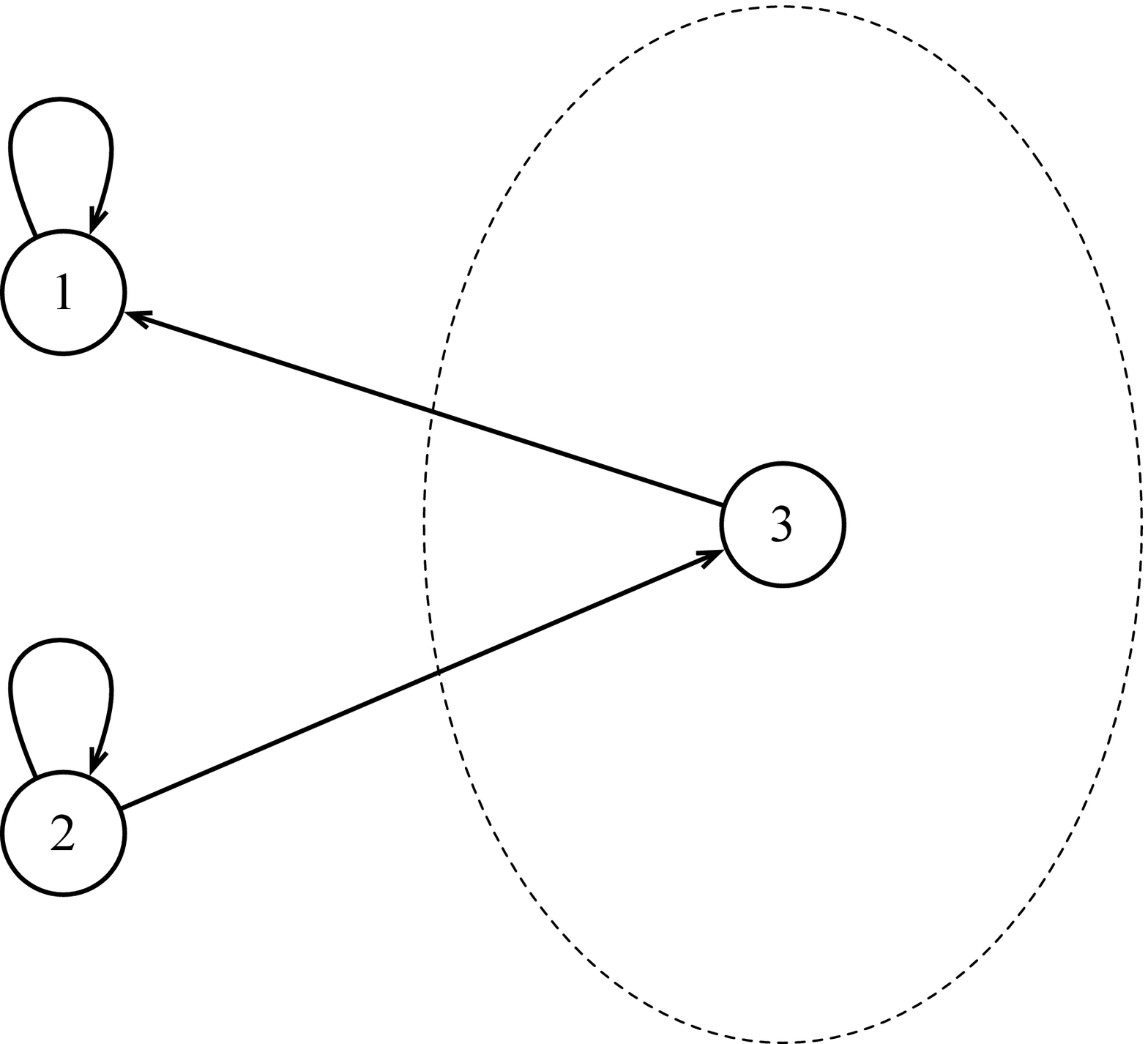}
\label{fig:tabh}}
\caption{Enumeration of subgraphs}
\label{fig:strongobservablecases}
\end{figure}

Figure~\ref{fig:strongobservablecases} shows all possibilities configurations of non-revealing, strongly-observable graphs $G=(V,E)$ with $\alpha(G)>1$.
The nodes 1 and 2 denote the vertices with no incoming edges from a common vertex. (By the proof of Lemma~\ref{lem:ind_num}, no edge exists between nodes 1 and 2, so each node must have a self-loop.) Furthermore, we can check that in each case, the subgraph on $V_1 = \{1,2\}$ preserves the observability of $G$.

Next, we show that the game induced by $G$ is at least as hard as the game induced by $G_1$:

\begin{lemma}\label{graphreductionlma}
Suppose the subgraph $G_1$ preserves the observability of $G$. Let $\cA$ and $\cA'$ denote the family of randomized player's strategy on feedback graph $G$ and $G_1$, respectively. Let $\cF$ and $\cF'$ be the family of deterministic (possibly adaptive) loss sequences on $G$ and  $G_1$, respectively. Finally, let $R_T$ and $R_T'$ denote the regret of the online learning problem induced by $G$ and $G_1$, respectively. We have
\begin{align}
\label{EqnGraphRed}
\begin{split}
    \min_{ \cA'}&\max_{\cF'}\bbE[R_T'] \le \min_{\cA}\max_{\cF}\bbE[R_T].
\end{split}
\end{align}
\end{lemma}

\begin{proof}
Let $X_t^G$ and $X_t^{G_1}$ denote the deterministic player's action at round $t$ in the online learning problem induced by the feedback graphs $G$ and $G_1$, respectively, and let $\ell_t(X_1^G,\ldots, X_t^G)$ and $\ell_t'(X_1^{G_1},\ldots, X_t^{G_1})$ denote the corresponding randomized loss picked by the (possibly adaptive) adversary on $G$ and $G_1$, respectively. According to Yao's minimax principle, It suffices to show that 
\begin{align*}
\max_{(\ell_t')_{t\in[T]}}\min_{X_{1:t}^{G_1}} \bbE\left[\sum_{t=1}^T\ell_t'(X_{1:t}^{G_1})-\min_{i\in V_1}\sum_{t=1}^T\ell_t'(\bi^t)\right] 
  \le \max_{(\ell_t)_{t \in [T]}} \min_{X_{1:t}^G} \bbE\left[\sum_{t=1}^T\ell_t(X_{1:t}^G)-\min_{i\in V}\sum_{t=1}^T\ell_t(\bi^t)\right],
%
\end{align*}
Consider the loss sequence $\{f'_t\}$ achieving the optimum on the left-hand side of inequality~\eqref{EqnGraphRed}.
Define the mapping function $g:V\mapsto V_1$ according to
\[
g(v)=\begin{cases}
v, & \text{if }v\in V_1,\\
v_{ob}, & \text{otherwise},
\end{cases}
\]
and define the loss function on $G$
\begin{align*}
f_t(X_1,\ldots,X_t)=\begin{cases}
f_t'(g(X_1),\ldots, g(X_t)), & \text{if } X_t\in V_1, \\
1, &\text{otherwise}.
\end{cases}
\end{align*}
With the above definitions, it is easy to see that 
\[
    \bbE\left[\min_{i\in V_1}\sum_{t=1}^T f_t'(\bi^t)\right]=\bbE\left[\min_{i\in V}\sum_{t=1}^T f_t(\bi^t)\right].
\]

Now let $X^*_{1:T}$ denote the optimal player strategy with respect to the loss functions $\{f_t\}$ on $G$. We may use Algorithm~\ref{alg:actiongen} to define a player strategy $X'_{1:T}$ on the subgraph $G_1$.

\begin{algorithm}[]
    \caption{Action Sequence Generation}
    \begin{algorithmic}[1]
    \renewcommand{\algorithmicrequire}{\textbf{Input:}}
    \renewcommand{\algorithmicensure}{\textbf{Output:}}
    \REQUIRE Time horizon $T$, feedback graph $G=(V,E)$, subgraph $G_1=(V_1,E_1)$, strategy $X^*_{1:T}$ on $G$
    \ENSURE Sequence of actions $X'_{1:T}$ on $G_1$
    \FOR {round $t=1,2\ldots$}
        \IF {$(X^*_t \in V_1)$}
            \STATE Set $X'_t = X^*_t$
        \ELSE 
            \STATE Set $X'_t = (X^*_t)_{ob}$
        \ENDIF
    \ENDFOR
    \RETURN $X'_{1:T}$
    \end{algorithmic} 
    \label{alg:actiongen}
\end{algorithm}
It is easy to see that the strategy $X'_{1:T}$ is indeed a valid strategy on $G_1$, since the next action only depends on the feedback observed with respect to the subgraph. More explicitly, let $Y^*_t$ (respectively, $Y_t'$) be the observed loss vector at round $t$ obtained by playing $X^*_t$ (respectively, $X_t'$). We pad the unobserved components of $Y^*_t$ and $Y_t'$ with $-1$'s, so that $|Y^*_t|=|V|$ and $|Y_t'|=|V_1|$, and renumber the nodes of the graph so that
\begin{equation*}
Y^*_t = \left((Y^*_t)^{(1)}, \; (Y^*_t)^{(2)}\right),
\end{equation*}
where $(Y^*_t)^{(1)}$ (respectively, $(Y^*_t)^{(2)}$) is composed of observations from nodes in $V_1$ (respectively, $V_2=V\backslash V_1$), and entry $i$ of $Y^*_t$ is given by $\mathbbm{1}_{i \in \No(X^*_{t})}\ell_t(X^*_{1:t-1};i)-\mathbbm{1}_{i\not\in \No(X^*_{t})}$.

Recall that an deterministic algorithm $A$ at round $t+1$ behaves as follows:
\[
 A: X^*_{1:t}\times Y^*_{1:t}\mapsto X^*_{t+1}
\]
Since the trailing $|V_2|$ coordinates of $Y^*_t$ are all $1$'s and $-1$'s, and they are deterministic given $X^*_{1:t}$, then $X^*_{1:t}\times (Y^*_{1:t})^{(1)}\mapsto X^*_{t+1}$ is a valid deterministic algorithm for $G$. By the observability preserving property of $G_1$ and our construction of $X_{1:T}'$, playing $X_{1:T}'$ on $G_1$ ensures that the observed entries in $Y'_t$ are a superset of the observed entries of $(Y^*_t)^{(1)}$. By design, $X'_{1:t}\times (Y^*_{1:t})^{(1)}\mapsto X'_{t+1}$ is a valid deterministic strategy for $G$. Therefore, $X'_{1:t}\times Y'_{1:t}\mapsto X'_{t+1}$ is a deterministic function, and the sequence $X'_{1:T}$ is indeed a valid player strategy for $G_1$.

Furthermore, we have
\begin{equation*}
f'_t(X'_{1:t}) \leq f_t(X^*_{1:t}), \qquad \forall t.
\end{equation*}
Putting the pieces together, we obtain
\begin{align*}
\max_{(\ell_t')_{t\in[T]}}\min_{X_{1:t}^{G_1}} \bbE\left[\sum_{t=1}^T\ell_t'(X_{1:t}^{G_1})-\min_{i\in V_1}\sum_{t=1}^T\ell_t'(\bi^t)\right] 
 & =\min_{X_{1:t}^{G_1}}\bbE\left[\sum_{t=1}^Tf_t'(X_{1:t}^{G_1})-\min_{i\in V_1}\sum_{t=1}^Tf_t'(\bi^t)\right] \\
& \le \bbE\left[\sum_{t=1}^T f_t'(X_{1:t}')-\min_{i\in V_1}\sum_{t=1}^T f_t'(\bi^t)\right] \\
& \le \bbE\left[\sum_{t=1}^T f_t(X_{1:t}^*)-\min_{i\in V}\sum_{t=1}^T f_t(\bi^t)\right] \\
& = \min_{X_{1:t}^G} \bbE\left[\sum_{t=1}^T f_t(X_{1:t}^G)-\min_{i\in V}\sum_{t=1}^T f_t(\bi^t)\right] \\
& \le \max_{(\ell_t)_{t \in [T]}} \min_{X_{1:t}^G} \bbE\left[\sum_{t=1}^T\ell_t(X_{1:t}^G)-\min_{i\in V}\sum_{t=1}^T\ell_t(\bi^t)\right]
%
\end{align*}
\end{proof}


%% file: arxiv_submission.bbl
\begin{thebibliography}{10}
\providecommand{\url}[1]{#1}
\csname url@samestyle\endcsname
\providecommand{\newblock}{\relax}
\providecommand{\bibinfo}[2]{#2}
\providecommand{\BIBentrySTDinterwordspacing}{\spaceskip=0pt\relax}
\providecommand{\BIBentryALTinterwordstretchfactor}{4}
\providecommand{\BIBentryALTinterwordspacing}{\spaceskip=\fontdimen2\font plus
\BIBentryALTinterwordstretchfactor\fontdimen3\font minus
  \fontdimen4\font\relax}
\providecommand{\BIBforeignlanguage}[2]{{%
\expandafter\ifx\csname l@#1\endcsname\relax
\typeout{** WARNING: IEEEtran.bst: No hyphenation pattern has been}%
\typeout{** loaded for the language `#1'. Using the pattern for}%
\typeout{** the default language instead.}%
\else
\language=\csname l@#1\endcsname
\fi
#2}}
\providecommand{\BIBdecl}{\relax}
\BIBdecl

\bibitem{cesa2006prediction}
N.~Cesa-Bianchi and G.~Lugosi, \emph{Prediction, Learning, and Games}.\hskip
  1em plus 0.5em minus 0.4em\relax Cambridge University Press, 2006.

\bibitem{mannor2011bandits}
S.~Mannor and O.~Shamir, ``From bandits to experts: {O}n the value of
  side-observations,'' in \emph{Advances in Neural Information Processing
  Systems}, 2011, pp. 684--692.

\bibitem{helmbold1992apple}
D.~P. Helmbold, N.~Littlestone, and P.~M. Long, ``Apple tasting and nearly
  one-sided learning,'' in \emph{Proceedings of the 33rd Annual Symposium on
  Foundations of Computer Science}.\hskip 1em plus 0.5em minus 0.4em\relax
  IEEE, 1992, pp. 493--502.

\bibitem{arora2012online}
R.~Arora, O.~Dekel, and A.~Tewari, ``Online bandit learning against an adaptive
  adversary: {F}rom regret to policy regret,'' \emph{arXiv preprint
  arXiv:1206.6400}, 2012.

\bibitem{cesa2013online}
N.~Cesa-Bianchi, O.~Dekel, and O.~Shamir, ``Online learning with switching
  costs and other adaptive adversaries,'' in \emph{Advances in Neural
  Information Processing Systems}, 2013, pp. 1160--1168.

\bibitem{dekel2014bandits}
O.~Dekel, J.~Ding, T.~Koren, and Y.~Peres, ``Bandits with switching costs:
  {$T^{2/3}$} regret,'' in \emph{Proceedings of the Forty-Sixth Annual ACM
  Symposium on Theory of Computing}.\hskip 1em plus 0.5em minus 0.4em\relax
  ACM, 2014, pp. 459--467.

\bibitem{alon2015online}
N.~Alon, N.~Cesa-Bianchi, O.~Dekel, and T.~Koren, ``Online learning with
  feedback graphs: {B}eyond bandits,'' in \emph{Conference on Learning Theory},
  2015, pp. 23--35.

\bibitem{alon2017nonstochastic}
N.~Alon, N.~Cesa-Bianchi, C.~Gentile, S.~Mannor, Y.~Mansour, and O.~Shamir,
  ``Nonstochastic multi-armed bandits with graph-structured feedback,''
  \emph{SIAM Journal on Computing}, vol.~46, no.~6, pp. 1785--1826, 2017.

\bibitem{bartok2014partial}
G.~Bart{\'o}k, D.~P. Foster, D.~P{\'a}l, A.~Rakhlin, and C.~Szepesv{\'a}ri,
  ``Partial monitoring: {C}lassification, regret bounds, and algorithms,''
  \emph{Mathematics of Operations Research}, vol.~39, no.~4, pp. 967--997,
  2014.

\bibitem{helmbold1997some}
D.~Helmbold and S.~Panizza, ``Some label efficient learning results,'' in
  \emph{Proceedings of the Tenth Annual Conference on Computational Learning
  Theory}.\hskip 1em plus 0.5em minus 0.4em\relax ACM, 1997, pp. 218--230.

\bibitem{cesa2005minimizing}
N.~Cesa-Bianchi, G.~Lugosi, and G.~Stoltz, ``Minimizing regret with label
  efficient prediction,'' \emph{IEEE Transactions on Information Theory},
  vol.~51, no.~6, pp. 2152--2162, 2005.

\bibitem{cesa2006regret}
------, ``Regret minimization under partial monitoring,'' \emph{Mathematics of
  Operations Research}, vol.~31, no.~3, pp. 562--580, 2006.

\bibitem{piccolboni2001discrete}
A.~Piccolboni and C.~Schindelhauer, ``Discrete prediction games with arbitrary
  feedback and loss,'' in \emph{Computational Learning Theory}.\hskip 1em plus
  0.5em minus 0.4em\relax Springer, 2001, pp. 208--223.

\bibitem{bartok2010toward}
G.~Bart{\'o}k, D.~P{\'a}l, and C.~Szepesv{\'a}ri, ``Toward a classification of
  finite partial-monitoring games,'' in \emph{International Conference on
  Algorithmic Learning Theory}.\hskip 1em plus 0.5em minus 0.4em\relax
  Springer, 2010, pp. 224--238.

\bibitem{antos2013toward}
A.~Antos, G.~Bart{\'o}k, D.~P{\'a}l, and C.~Szepesv{\'a}ri, ``Toward a
  classification of finite partial-monitoring games,'' \emph{Theoretical
  Computer Science}, vol. 473, pp. 77--99, 2013.

\bibitem{kalai2005efficient}
A.~Kalai and S.~Vempala, ``Efficient algorithms for online decision problems,''
  \emph{Journal of Computer and System Sciences}, vol.~71, no.~3, pp. 291--307,
  2005.

\bibitem{geulen2010regret}
S.~Geulen, B.~V{\"o}cking, and M.~Winkler, ``Regret minimization for online
  buffering problems using the weighted majority algorithm.'' in \emph{COLT},
  2010, pp. 132--143.

\bibitem{Yao77}
A.~Yao, ``Probabilistic computations: {T}oward a unified measure of
  complexity,'' in \emph{Proceedings of the 18th Annual Symposium on
  Foundations of Computer Science}.\hskip 1em plus 0.5em minus 0.4em\relax
  IEEE, 1977, pp. 222--227.

\end{thebibliography}
